\documentclass[11pt]{amsart}
\usepackage[margin=1.3in]{geometry}                % See geometry.pdf to learn the layout options. There are lots.
\usepackage{graphicx}
\usepackage{amssymb}
\usepackage{epstopdf}
\usepackage{amsmath}
\usepackage{mathtools}
\usepackage{hyperref}
\usepackage{subcaption}

\usepackage[textsize=tiny]{todonotes}
\newtheorem{lemma}{Lemma}[section]
\newtheorem{remark}{Remark}[section]
\newtheorem{example}{\bf Example}[section]

\definecolor{darkred}{rgb}{.7,0,0}

\definecolor{darkgreen}{rgb}{.15,.55,0}

\definecolor{darkblue}{rgb}{0,0,0.7}

\usepackage{url}
\usepackage[square,numbers]{natbib}

\newcommand*{\R}{{\mathbb{R}}}

\title{Stable generative modeling using Schr\"odinger bridges}
\author{Georg A.~Gottwald \and Fengyi Li \and Youssef Marzouk \and Sebastian Reich}
\date{\today}                                           % Activate to display a given date or no date

\begin{document}

%%%% Abstract text to be placed here %%%%%%%%%%%%
\begin{abstract}
We consider the problem of sampling from an unknown distribution for which only a sufficiently large number of training samples are available. Such settings have recently drawn considerable interest in the context of generative modelling and Bayesian inference. In this paper, we propose a generative model combining Schr\"odinger bridges and Langevin dynamics. Schr\"odinger bridges over an appropriate reversible reference process are used to approximate the conditional transition probability from the available training samples, which is then implemented in a discrete-time reversible Langevin sampler to generate new samples. By setting the kernel bandwidth in the reference process to match the time step size used in the unadjusted Langevin algorithm, our method effectively circumvents any stability issues typically associated with the time-stepping of stiff stochastic differential equations. Moreover, we introduce a novel split-step scheme, ensuring that the generated samples remain within the convex hull of the training samples. Our framework can be naturally extended to generate conditional samples and to Bayesian inference problems. We demonstrate the performance of our proposed scheme through experiments on synthetic datasets with increasing dimensions and on a stochastic subgrid-scale parametrization conditional sampling problem as well as generating sample trajectories of a dynamical system using conditional sampling.
\end{abstract}
%%%%%%%%%%%%%%%%%%%%%%%%%%%

%%%%%%%%%% Insert the texts which can accomdate on firstpage in the tag "fmtext" %%%%%

%%%%%%%%%%%%%%% End of first page %%%%%%%%%%%%%%%%%%%%%

\maketitle

\section{Introduction}

Generative modeling is the process of learning a mechanism for synthesizing new samples that resemble those of the original data-generating distribution, given only a finite set of samples. It has seen wide adoption and enormous success across diverse application domains, from image \cite{cv1,cv2} and text generation \cite{nlp1,nlp2}, to drug discovery \cite{drug1, drug2} and anomaly detection \cite{anomaly1,anomaly2}, to name but a few. 

In this paper, we introduce a new nonparametric approach to generative modeling that combines ideas from Schr\"odinger bridges \cite{PeyreCuturi,CGP21} and reversible Langevin dynamics \cite{Pavliotis2016}. Suppose that we are given $M$ training samples $x^{(i)} \sim \pi$, $i=1,\ldots,M$, from an unknown distribution $\pi$ on $\mathbb{R}^d$. Perhaps the simplest nonparametric approach to generative modeling is to build a kernel density estimate (KDE) and then sample from it; the KDE is essentially a mixture model with $M$ components. Alternatively, one could estimate the score function $s(x) = \nabla \log \pi(x)$, without directly estimating $\pi(x)$, and use this estimate in the Langevin dynamics
\begin{equation} \label{eq:score_based}
\dot{X}_\tau = s(X_\tau) + \sqrt{2}\dot{W}_\tau,
\end{equation}
where $W_\tau$ denotes standard $d$-dimensional Brownian motion \cite{Pavliotis2016}. There is a plethora of ways of estimating the score function \cite{hyvarinen2005estimation,song2020sliced}, and given an estimate for it, one needs to discretize \eqref{eq:score_based}, for example using Euler--Maruyama, to obtain an implementable scheme. However, the step size needs to be carefully chosen: a small step size leads to slow convergence, while too large a step yields instability of the numerical scheme, especially for data that are supported on or are strongly concentrated about a compact sub-manifold of $\mathbb{R}^d$, e.g.,
\begin{equation} \label{eq:manifold}
\mathcal{M} = \{x \in \mathbb{R}^d: g(x) =0\},
\end{equation}
for some unknown function $g(x)$; a situation commonly encountered in high-dimensional data and referred to as the manifold hypothesis \cite{FeffermanEtAl16,WhiteleyEtAl24} and which is known to be challenging for generative models based on (\ref{eq:score_based}) \cite{diffusion2}. Any estimated score function, denoted here by $\hat s_M(x)$, will take large values for $x$ with $\|g(x)\|^2 \gg 0$, rendering the Langevin dynamics \eqref{eq:score_based} stiff. This implies that an explicit time integration method such as Euler--Maruyama will require very small step sizes $\Delta \tau >0$.

\begin{example} \label{ex:2D}
In order to illustrate this point consider the singular Gaussian distribution on $\mathbb{R}^2$ with $x=(x_1,x_2) \in \mathbb{R}^2$ satisfying $x_1 \sim {\rm N}(0,1)$ and $x_2 = 0$. Let us assume that the data samples $x^{(i)}$, $i=1,\ldots,M$, have been polluted by noise from a Gaussian in $\mathbb{R}^2$ with mean zero and covariance $\nu I$ for $0< \nu \ll 1$. A standard score estimator $\hat s_M(x)$ will lead to
\begin{equation}
    \hat s_\infty (x) = \left( \begin{array}{c} -\frac{1}{1+\nu} x_1\\ -\frac{1}{\nu}x_2\end{array} \right)
\end{equation}
in the limit $M\to \infty$ and an Euler--Maruyama discretization of the corresponding Langevin dynamics (\ref{eq:score_based}) will require step-sizes $\Delta \tau < 2\nu$, which become arbitrarily small as $\nu \to 0$. We will return to this simple problem in Example \ref{ex:2Db} where it will be used to demonstrate certain advantages of the methodology proposed in this paper and, in particular, that we can circumvent the just described computational bottleneck as $\nu \to 0$.
\end{example}

\noindent
Motivated by this illustrative example, we follow an alternative approach in this paper. Namely, instead of first estimating the score function $s(x)$ and then discretising the corresponding approximation to~\eqref{eq:score_based} in time, we employ Schr\"odinger bridges \cite{PeyreCuturi,CGP21} to directly estimate the conditional expectation value 
\begin{equation}
\mu (x;\epsilon) := \mathbb{E}[X_\epsilon|X_0=x]
\end{equation}
from the given samples $\{x^{(i)}\}_{i = 1}^M$ for given parameter $\epsilon >0$. We denote the data-driven estimator by $m(x;\epsilon):\mathbb{R}^d \to \mathbb{R}^d$. 

More specifically, we first note that $\mu(x;\epsilon) = \exp(\epsilon \mathcal{L})\,{\mbox{id}}(x)$, where ${\mbox{id}}:\mathbb{R}^d \to \mathbb{R}^d$ denotes the identity map and $\mathcal{L}$ the generator of the Langevin dynamics (\ref{eq:score_based}) \cite{Pavliotis2016}. We then approximate the semi-group $\exp(\epsilon \mathcal{L})$ from the given samples $\{x^{(i)}\}_{i=1}^M$ using a Schr\"odinger bridge approximation \cite{MarshallCoifman,WR20}, which optimally couples the empirical measure of the training samples with itself over an appropriate random walk reference process. By solving the Schr\"odinger bridge problem, we construct a transition matrix whose state space encompasses all the training points, which in turn defines $m(x;\epsilon)$ for all $x = x^{(i)}$, $i=1,\ldots,M$. We then extend this approximation beyond the training data to all $x \in \mathbb{R}^d$,  which leads to the desired approximation $m(x;\epsilon)$.

The second key ingredient of our method is to interpret $\epsilon$ as a step size and to read off a Gaussian transition kernel from the Schr\"odinger bridge approximation in the form of
\begin{equation} \label{eq:propagator_general}
X_{n+1} = m(X_n;\epsilon) + \sqrt{\Sigma (X_n)} \,\Xi_n,
\end{equation}
with appropriately defined diffusion matrix $\Sigma:\mathbb{R}^d \to \mathbb{R}^{d\times d}$ and $\Xi_n \sim {\rm N}(0,\epsilon I)$. Broadly, $m(X_n;\epsilon)$ controls the drift, while $\sqrt{\Sigma(X_n)}\,\Xi_n$ introduces noise. An obvious choice for $\Sigma(x)$ 
is $\Sigma(x) = 2I$, which corresponds to (\ref{eq:score_based}) and its Euler--Maruyama discretisation with step-size $\Delta \tau = \epsilon$. In addition, we explore data-informed choices of $\Sigma(x)$ in (\ref{eq:propagator_general}).

Comparing to directly discretizing~\eqref{eq:score_based} using Euler--Maruyama with step-size $\Delta \tau = \epsilon$, we will demonstrate that the scheme~\eqref{eq:propagator_general} is stable and ergodic for all step-sizes $\epsilon >0$ and, hence, $\epsilon$ can be chosen solely on accuracy considerations. We also introduce a novel split-step time-stepping scheme, which ensures that the generated samples lie in the convex hull of the training samples.

In addition, we replace the constant diffusion matrix $\Sigma (x) = 2 I$ with a scaled matrix
\begin{equation} \label{eq:scaled covariance}
\Sigma (x) = 2 \rho(x) I
\end{equation}
for given bandwidth $\rho(x)>0$, which requires appropriate modifications to the Schr\"odinger bridges considered in \cite{WR20}. As we demonstrate in our numerical experiments, the resulting sampling scheme~\eqref{eq:propagator_general} provides a better representation of the underlying target distribution.  More precisely, we assess the quality of the generated samples using a variable bandwidth kernel and using a fixed bandwidth kernel on synthetic data sets drawn from non-uniform distributions supported on irregular domains and on low-dimensional manifolds. 

We then extend our method to cover Bayesian inference problems with $\pi(x)$ as prior and to create a conditional generative model. These extensions allow us to perform Bayesian inference in the ``simulation-based'' setting, i.e., without explicit evaluation of a prior density and, in the case of conditional sampling, even without evaluations of the data likelihood.  We demonstrate the performance of our conditional generative model for a stochastic subgrid-scale parametrization problem and for the generation of synthetic time series of dynamical systems. 

%%%%%%%%%%%%%%%%%%%%%%%%
% 
\subsection{Related work}
%
%%%%%%%%%%%%%%%%%%%%
Langevin dynamics~\eqref{eq:score_based} characterizes the motion of particles as they experience a blend of deterministic and stochastic forces. Unlike in this paper, it is typically assumed that the deterministic forcing term $\nabla \log \pi(x)$ is given. Langevin dynamics has become a popular tool for sampling data from the target distribution $\pi(x)$. One variation of this is to introduce a symmetric preconditioning operator to the Langevin dynamics and to consider reversible processes of the general form
\begin{equation} \label{eq:reversible diffusion}
\dot{X}_\tau = K(X_\tau) \nabla \log \pi(X_\tau) + \nabla \cdot K(X_\tau) + \sqrt{2K(X_\tau)}\dot{W}_\tau,
\end{equation}
which samples from the distribution $\pi(x)$ for any symmetric positive definite matrix $K(x)$. We adopt here the It\^o interpretation of the multiplicative noise term \cite{Pavliotis2016}. While (\ref{eq:reversible diffusion}) has originally been discussed in molecular physics \cite{Fixman78,ErmakMcCammon,HutterOttinger98}, popular choices of $K(x)$ arising from computational statistics include the empirical covariance \cite{Garbuno-Inigo} and the Riemannian metric \cite{RMLD1, Li2015PreconditionedSG}, making this method converge faster and more geometry-aware, while leaving the stationary distribution unchanged. The scaled diffusion matrix (\ref{eq:scaled covariance}) corresponds to the choice $K(x) = \rho(x) I$. Optimal choice of $K(x)$ in terms of convergence to equilibrium have recently been discussed in \cite{LelievreEtAl24}.

Our approximation of the semi-group $\exp(\epsilon \mathcal{L})$ relies on recent work on diffusion maps and an accelerated Sinkhorn algorithm \cite{MarshallCoifman,WR20}. The Sinkhorn algorithm solves for the Markov transition kernel associated with a discrete Schr\"odinger bridge problem, where the coupling is between the empirical measure of the training samples with itself. This approach results in a symmetric bi-stochastic matrix that, notably, approximates the semi-group $\exp(\epsilon \mathcal{L})$ to higher accuracy in $\epsilon$ than standard diffusion map approximations \cite{diffusion_map1,diffusion_map2}. Separately, the idea of using variable bandwidth kernels can be found very early in the statistics community, for example, in the context of kernel density estimation \cite{vb_kde_Salgado-Ugarte, terrell1992variable}. Recently, \cite{DM_variable_bw} replaces the original fixed bandwidth kernel with the variable bandwidth kernel in the construction of diffusion maps, making the approximation of the generator accurate on unbounded domains. Inspired by this concept, we replace the fixed bandwidth kernel $\Sigma = 2I$ with a variable bandwidth kernel (\ref{eq:scaled covariance}) for given $\rho(x)>0$. The resulting Schr\"odinger bridge approximates the semi-group of the reversible Langevin diffusion process (\ref{eq:reversible diffusion}) with $K(x) = \rho(x) I$. 

Several recent studies have combined a range of score function estimation techniques with Langevin dynamics. For example,  \cite{diffusion2} introduces a noise conditional score network to learn the score function and then uses annealed Langevin dynamics to generate samples. \cite{Block2020GenerativeMW} studies the convergence rate of a Langevin based generative model, where the score is estimated using denoising auto-encoders. Such techniques are also studied within the Bayesian imaging community, commonly referred to as ``plug and play'' \cite{plugandplay}. These approaches typically use neural networks for the estimation of the score function. The Schr\"odinger bridge sampler proposed in this paper is instead built upon a direct discrete-time approximation to (\ref{eq:score_based}). In addition, \cite{plugandplay} uses an explicit projection to ensure that samples stay on the compact manifold given by~\eqref{eq:manifold} which is assumed to be explicitly known. In contrast, we do not assume any knowledge of $\mathcal{M}$.

We finally mention diffusion models or score generative models (SGM), which have been successfully used for generative modelling, in particular in image generation \cite{diffusion1, diffusion2}. These methods solve both a forward and a reverse stochastic differential equation (SDE). The forward SDE introduces noise to the sample, evolving the prior into the standard normal distribution, while the reverse SDE evolves samples from the standard normal distribution back to the original data distribution, yielding a different sample than the one initially fed into the forward SDE. During the training process, the score function is learned for the nonstationary distribution at each time. We mention that Schr\"odinger bridges have also been implemented in the context of SGMs in order to exactly couple the target distribution with the standard normal distribution \cite{VargasEtAl21,DeBortoliEtAl21,ShiEtAl23,ChenEtAl22,Peluchetti24}. Here we consider a completely different application of Schr\"odinger bridges; namely the estimation of conditional expectation values from training data. 

%%%%%%%%%%%%%%%%%%%%%%%%%
%
\subsection{Outline}
%
%%%%%%%%%%%%%%%%%%%%%%%%%%
In Section~\ref{sec:discrete SB} we construct a Markov chain using a Schr\"odinger bridge approximation that samples from the given discrete data distribution. In Section~\ref{sec:diffusion approximation}, we extend this Markov chain to the continuous state space setting by constructing a Gaussian transition kernel which
extracts its conditional mean and covariance matrix from the underlying diffusion map approximation. We introduce two discrete-time Langevin samplers; one with a data-unaware diffusion and one with a data-aware diffusion matrix in Section~\ref{sec:3.1}. Theoretical properties such as stability and ergodicity are discussed in Section~\ref{sec:3.2}. We further discuss the application of variable bandwidth kernels when constructing the Schr\"odinger bridge in Section~\ref{sec:vb}. Here the goal is to approximate the tails of the reference distribution $\pi(x)$ better from the available training samples. In terms of practical applications, we explore the extension of our proposed scheme to a conditional sampling setting and Bayesian inference in Section~\ref{sec:conditional}. We demonstrate our proposed methods in  Section~\ref{sec:numerics} in a suite of examples. We start with a couple of synthetic examples demonstrating among others the benefit of variable bandwidth implementations. In terms of applications, we employ conditional sampling to provide a data-driven stochastic subgrid-scale parametrization for multi-scale systems. Furthermore, we show how our approach can be used to generate realistic synthetic trajectories of a dynamical system, given only a single time series for training. 
We conclude in  Section~\ref{sec:conclusion} with a summary and an outlook. 

%%%%%%%%%%%%%%%%%%%%%%%%%%%%
%
\section{Discrete Schr\"odinger bridges}\label{sec:discrete SB}
%
%%%%%%%%%%%%%%%%%%%%%%%%%%%%
In this section, we collect some preliminary building blocks by considering the simpler task of building a discrete Markov chain over the samples $\{x^{(i)}\}_{i=1}^M$, which leaves the associated empirical probability measure
\begin{equation}
\mu_{\rm em} ({\rm d}x) = \frac{1}{M} \sum_{i=1}^M  \delta_{x^{(i)}}({\rm d}x)
\end{equation}
in $\mathbb{R}^d$ invariant. Here $\delta_x ({\rm d}x)$ denotes the Dirac delta distribution centred at $x$. In the subsequent section, we will generalise the finding from this section to approximately sample from $\pi(x)$, allowing for the generation of new samples which are distinct from the given training samples.

We consider the Schr\"odinger bridge problem of coupling $\mu_{\rm em}({\rm d}x)$ with itself along a reversible reference process 
with (unnormalized) transition probabilities
\begin{equation} \label{eq:tij}
t_{ij} = \exp \left( -\frac{1}{2 \epsilon} (x^{(i)}-x^{(j)})^{\top} 
\left(K(x^{(i)}) + K(x^{(j)})\right)^{-1}
(x^{(i)}-x^{(j)}) \right),
\end{equation}
which we collect into a symmetric matrix $T \in \mathbb{R}^{M\times M}$. Here $\epsilon >0$ is a tuneable parameter and $K(x)$ is a symmetric positive definite matrix for all $x \in \mathbb{R}^d$. Popular choices include $K= I$, $K = \Sigma_{M}$, where $\Sigma_M$ is the empirical covariance matrix of the training samples $\{x^{(i)}\}_{i=1}^M$, and $K = \rho(x)I$, where $\rho (x)>0$ is a scaling function representing variable bandwidth.  

Instead of working with the empirical measure $\mu_{\rm em}({\rm d}x)$, we introduce the probability  vector 
$p^\ast = (1/M,\ldots,1/M)^{\top} \in \mathbb{R}^M$ over $\{x^{(i)}\}_{i=1}^M$. Then the associated Schr\"odinger bridge problem can be reformulated into finding the non-negative scaling vector $v \in \mathbb{R}^M$ such that the symmetric matrix
\begin{equation} \label{eq:DM_approximation}
P = D(v) T D(v)
\end{equation}
is a Markov chain with invariant distribution $p^\ast$, i.e.,
\begin{equation}
P p^\ast = p^\ast.
\end{equation}
Here $D(v) \in \mathbb{R}^{M\times M}$ denotes the diagonal matrix with diagonal entries provided by $v \in \mathbb{R}^M$. We remark that the standard scaling used in Schr\"odinger bridges would lead to a bi-stochastic matrix $\tilde P$, which is related to (\ref{eq:DM_approximation}) by $\tilde P = M^{-1} P$.

Given $P$, one can now construct a Monte Carlo scheme that samples from $\mu_{\rm em}$. Assume the Markov chain is currently in state $x^{(j)}$, then the transition probabilities to the next state $x \in \{x^{(i)}\}_{i=1}^M$ are given by 
\begin{equation}
p_j = P e_j \in \mathbb{R}^M,
\end{equation}
where $e_j \in \mathbb{R}^M$ denotes the $j$-th unit vector in $\mathbb{R}^M$. Since all entries in $P$ are bounded from below provided all samples satisfy $x^{(i)} \in \mathcal{C}$, where $\mathcal{C}$ is a compact subset in $\mathbb{R}^d$, the constructed Monte Carlo scheme possesses a unique invariant measure given by $p^\ast$ and is geometrically ergodic. The rate of convergence can be determined by the diffusion distance
\begin{equation}
    d(x^{(i)},x^{(j)}) = \|p_i - p_j\|^2.
\end{equation}
If the diffusion distance is small, then $x^{(i)}$ and $x^{(j)}$ are well connected. Furthermore, if $d(x^{(i)},x^{(j)})$ is small for all points, then the Markov chain will mix quickly. In particular, larger values of $\epsilon$ will lead to faster mixing. 

However, the goal is to approximately sample from the underlying distribution $\pi(x)$ and not just the empirical distribution $\mu_{\rm em}({\rm d}x)$. The required extension of our baseline algorithm is discussed in the following section.

\section{Approximating the conditional mean} \label{sec:diffusion approximation}
%
%%%%%%%%%%%%%%%%%%%%%%%%%%%%
In order to implement~\eqref{eq:propagator_general}, we need to define $m(x;\epsilon)$ and $\Sigma(x)$ for any $x\in\mathbb R^d$. In this section, we discuss how one can obtain these functions from the training samples $\{x^{(i)}\}_{i=1}^M$ and the Markov chain approximation (\ref{eq:DM_approximation}).

We introduce the vector-valued function $t(x) \in \mathbb{R}^M$ with entries
\begin{equation}
\label{eq:tvec}
t_i(x) =  \exp \left( -\frac{1}{2\epsilon} (x^{(i)}-x)^{\top} \left(K(x)+K(x^{(i)})\right)^{-1}(x^{(i)}-x) \right)
\end{equation}
for $i = 1,\ldots,M$. We then define the probability vector using the Sinkhorn weights, $v$, obtained in~\eqref{eq:DM_approximation}, i.e., 
\begin{equation} \label{eq:probability vectors}
p (x) = \frac{D(v) t(x)}{v^{\top} t (x)} \in \mathbb{R}^M
\end{equation}
for all $x\in \mathbb{R}^d$. This vector gives the transition probabilities from $x$ to training samples $\{x^{(i)}\}_{i=1}^M$ and provides a finite-dimensional approximation to the conditional probability distribution $\pi_\epsilon (\cdot|x)$ of the true underlying diffusion process; i.e., the semi-group $\exp(\epsilon \mathcal{L})$ with generator $\mathcal{L}$ corresponding to the generalized reversible diffusion process (\ref{eq:reversible diffusion}). 

Finally, using the probability vector $p(x)$, which implicitly depends on $\epsilon$, and introducing the data matrix of training samples
\begin{equation}\label{eq:data_vec}
\mathcal{X} = (x^{(1)},\ldots,x^{(M)}) \in \mathbb{R}^{d\times M},
\end{equation}
our sample-based approximation of the conditional mean is provided by
\begin{equation} \label{eq:mean}
m(x;\epsilon) := \mathcal{X} p (x).
\end{equation}

%%%%%%%%%%%%%%%%%%%%%%%%%%%%%%%%%%%
\begin{remark}
The construction of the conditional mean $m(x;\epsilon)$ is known as the barycentric projection of the entropy-optimally coupling \cite{SeguyEtAl2018,PooladianNilesWeed2022}. In optimal transport, $v$ plays the role of the optimizer of the dual problem. We also mention the connection to de-noising schemes which utilize a structure similar to (\ref{eq:mean}). See the recent review \cite{MD24}. 
\end{remark}
%%%%%%%%%%%%%%%%%%%%%%%%%%%%%%%% 

%%%%%%%%%%%%%%%%%%%%%%%%%
%
\subsection{Sampling algorithms} \label{sec:3.1}
%
%%%%%%%%%%%%%%%%%%%%%%%
We now present our main Langevin sampling strategies based on the previously introduced probability vector $p (x)$ in~\eqref{eq:probability vectors}. Sampling schemes of the form (\ref{eq:propagator_general}) have the same drift term (\ref{eq:mean}), but differ in the way the diffusion matrix $\Sigma(x)$ is defined. We consider a data-unaware diffusion as well as a data-aware diffusion which turns out to be advantageous in generating new samples from the data distribution $\pi$ (see the numerical experiments in Section~\ref{sec:numerics}).

%%%%%%%%%%%%%%%%%%%%%%%%%%%
%
\subsubsection{Langevin sampler with data-unaware diffusion}
%
%%%%%%%%%%%%%%%%%%%%%%%%%%%%%%%
Using $m_\epsilon(x)$, we propose the recursive sampler 
\begin{equation} \label{eq:update_tau}
X_{n+1} = X_n + \Delta \tau \left( \frac{m(X_n;\epsilon)-X_n}{\epsilon}\right) + \sqrt{2 K(X_n)} \Xi_n
\end{equation}
as an approximation to (\ref{eq:reversible diffusion}), where $\Delta \tau$ is the time step and $\Xi_n \sim {\rm N}(0,\Delta \tau I)$. If $K = I$,  we obtain the score function approximation
\begin{equation} \label{eq:score_approximation}
s (x) = \frac{m(x;\epsilon)-x}{\epsilon}
\end{equation}
in~\eqref{eq:score_based}. Furthermore, by taking $\Delta \tau = \epsilon$, we have
\begin{equation} \label{eq:update}
X_{n+1} = m(X_n;\epsilon) + \sqrt{2 K(X_n)} \Xi_n. 
\end{equation}
Note that~\eqref{eq:update} fits into the general formulation~\eqref{eq:propagator_general}
with $\Sigma(x) = 2K(x)$. 

Let us briefly discuss the qualitative behavior of the time-stepping method (\ref{eq:update}) as a function of $\epsilon>0$. For large $\epsilon$, the expected value $m(x;\epsilon)$ will become essentially independent of the current state $X_n$ and the diffusion process will sample from a centred Gaussian. For $\epsilon \to 0$, on the other hand, the probability vector $p(x)$ can potentially degenerate into a vector with a single entry approaching one with all other entries essentially becoming zero. Hence a key algorithmic challenge is to find a good value for $\epsilon$ and a suitable $K(x)$, which guarantee both good mixing and accuracy, i.e., $X_n \sim \pi$ as $n\to \infty$. 

In terms of initialisation, it is often best to initialise from one of the training data $x^{(j^\ast)}$ with
$j^\ast \in \{1,\ldots,M\}$ chosen at random.

%%%%%%%%%%%%%%%%%%%%%%%%%%%%%
%
\subsubsection{Langevin sampler with data-aware diffusion}
%
%%%%%%%%%%%%%%%%%%%%%%%%%%%%%
From~\eqref{eq:probability vectors} and~\eqref{eq:data_vec}, one can also define the scaled conditional covariance matrix, 
\begin{equation} \label{eq:cm_estimate}
C(x) = \epsilon^{-1} 
(\mathcal{X}-m(x;\epsilon) 1_M^{\top}) D(p(x))  (\mathcal{X}-m(x;\epsilon) 1_M^{\top})^{\top} \in
\mathbb{R}^{d\times d},
\end{equation}
which is the (scaled) covariance matrix associated with the probability vector $p(x)$. Here $1_M \in \mathbb{R}^M$ denotes the $M$-dimensional vector of ones. Therefore, one can more directly implement a Gaussian approximation associated with the transition probabilities $p(x)$ and introduce the update
\begin{equation} \label{eq:update2_tau}
X_{n+1} = X_n + \Delta \tau \left( \frac{m(X_n;\epsilon)-X_n}{\epsilon}\right) +  \sqrt{C(X_n)} \Xi_n.
\end{equation}
Similar to the previous case, setting $\Delta \tau = \epsilon$ implies
\begin{equation} \label{eq:update2}
X_{n+1} = m(X_n;\epsilon) + \sqrt{C(X_n)} \Xi_n,
\end{equation}
which we found to work rather well in our numerical experiments since it directly captures the uncertainty contained in the data-driven coupling $P$. The scheme~\eqref{eq:update2} corresponds to setting $\Sigma(x) =  C(x)$ in~\eqref{eq:propagator_general}. Also note that the scheme~\eqref{eq:update2} still depends on $K(x)$ through the probability vector $p(x)$. 

%%%%%%%%%%%%%%%%%%%%%%%%%%%%%%%%%%%%%%%%%%%
%
\subsection{Algorithmic properties} \label{sec:3.2}
%
%%%%%%%%%%%%%%%%%%%%%%%%%%%%%%%%%%%%%%%%
We briefly discuss several important properties on the stability and the ergodicity of the proposed Langevin samplers. The following Lemma establishes that, since each $p$ is a probability vector, $m(x;\epsilon)=\mathcal{X} p(x)$ is a convex combination of the training sample $\{x^{(i)}\}_{i = 1}^M$.

%%%%%%%%%%%%%%%%%%%%%%%%%%%%%%%%%%%%
\begin{lemma} \label{lemma1} 
Let us denote the convex hull generated by the data points 
$\{ x^{(i)}\}_{i=1}^M$ by $\mathcal{C}_M$. It holds that 
\begin{equation} \label{eq:stability}
    m(x;\epsilon) \in \mathcal{C}_M
\end{equation}
for all choices of $\epsilon > 0$ and all $x \in \mathbb{R}^d$.
\end{lemma}
%%%%%%%%%%%%%%%%%%%%%%%%%%%%%%%%%

%%%%%%%%%%%%%%%%%%%%%%%%%%%%%%%
\begin{proof}
    The Lemma follows from the definition~\eqref{eq:mean}, which we write as $m(x;\epsilon)=\sum_{j=1}^M x^{(j)} p_{j}(x)$,  and the fact that $p(x)$ is a probability vector with $0\le p_{j}(x)\le 1$ for all $\epsilon>0$ and all $x \in \mathbb{R}^d$.
\end{proof}
%%%%%%%%%%%%%%%%%%%%%%%%%%%%%%

\noindent
This establishes stability of the Langevin samplers~\eqref{eq:update} and~\eqref{eq:update2} for all step-sizes $\epsilon>0$. The next lemma shows that the Langevin sampler~\eqref{eq:update} is geometrically ergodic.

%%%%%%%%%%%%%%%%%%%%%%%%%%%%%
\begin{lemma} \label{lemma2} 
Let us assume that the data generating density $\pi(x)$ has compact support. Then the time-stepping method~\eqref{eq:update} possesses a unique invariant measure and is geometrically ergodic provided the norm of the symmetric positive matrix $K(x)$ is bounded from above and below for all $x \in \mathbb{R}^d$. 
\end{lemma}
%%%%%%%%%%%%%%%%%%%%%%%%%%%%%

%%%%%%%%%%%%%%%%%%%%%%%%%%%%%%%%
\begin{proof} We consider $K(x) = I$ for simplicity. We introduce the Lyapunov function $V(x) = \|x\|^2$ and balls
\begin{equation}
\mathcal{B}_R = \{ x\in \mathbb{R}^d : \|x\| \le R\}
\end{equation}
of radius $R>0$ in $\mathbb{R}^d$. Since $m_{\epsilon}(X_n) \in \mathcal{C}_M$ and $\pi(x)$ has compact support, one can find radii $R_\ast>0$ and $R> \sqrt{R_\ast^2 + 2 \epsilon}$, which are independent of the training data $\{x^{(i)}\}_{i=1}^M$, such that $\mathcal{C}_M \subset \mathcal{B}_{R_\ast}$ and
\begin{equation}
    \mathbb{E}[V(X_{n+1})|X_n] \le \lambda V(X_n)
\end{equation}
for all $X_n \notin \mathcal{B}_R$ and suitable $0\le \lambda < 1$. This follows from the fact that $\mathbb{E}[V(X_{n+1})|X_n] < R_\ast^2 + 2\epsilon$, while $V(x)>R^2$ for all $x\notin \mathcal{B}_R$. One then chooses $\lambda = (R_\ast^2 + 2 \epsilon)/R^2< 1$. Furthermore, there is a constant $\delta>0$ such that
\begin{equation*}
{\rm n}(x';m(x;\epsilon),2\epsilon I) \ge \delta
\end{equation*}
for all $x,x'\in \mathcal{B}_R$. Here ${\rm n}(x;m,\Sigma)$ denotes the Gaussian probability density function with mean $m$ and covariance matrix $\Sigma$. In other words, $\mathcal{B}_R$ is a small set in the sense of \cite{MeynTweedy}. Geometric ergodicity follows from Theorem 15.0.1 in \cite{MeynTweedy}. See also the self-contained presentation in \cite{MSH02}.
\end{proof}
%%%%%%%%%%%%%%%%%%%%%%%%%%%%%%%%%%%

\noindent
We note that extending Lemma \ref{lemma2} to the time-stepping scheme~\eqref{eq:update2} with a data-aware diffusion is non-trivial since the covariance matrix~\eqref{eq:cm_estimate} may become singular.

\begin{lemma} If $K(x)= I$, the conditional mean estimator $m(x;\epsilon)$ is equivalent to
\begin{equation} \label{eq:gradient log density estimator}
m (x;\epsilon) = x + \epsilon \nabla_x \log \Pi(x;\epsilon), \qquad 
\Pi (x;\epsilon) := (v^{\rm T}t(x))^2.
\end{equation}
Hence, the update (\ref{eq:update}) is equivalent to an Euler--Maruyama discretization of Langevin dynamics (\ref{eq:score_based}) with modified probability density 
\begin{equation}
    \tilde \pi(x;\epsilon) := \frac{\Pi(x;\epsilon)}{\int \Pi(x;\epsilon)\,{\rm d}x}.
\end{equation}
In other words, for $\epsilon$ sufficiently small and $K(x) = I$, (\ref{eq:update}) samples approximately from $\tilde \pi(x;\epsilon)$.
\end{lemma}

\begin{proof}
    Formula (\ref{eq:gradient log density estimator}) follows from (\ref{eq:probability vectors}) and 
    (\ref{eq:tvec}).
\end{proof}

\noindent
Lemma \ref{lemma1} suggests to replace the sampling step~\eqref{eq:update} by the associated split-step scheme
\begin{subequations} \label{eq:update_ss}
    \begin{align}
        X_{n+1/2} &= X_n + \sqrt{2K(X_n)} \Xi_n,\\
        X_{n+1} &=  m(X_{n+1/2};\epsilon).
    \end{align}
\end{subequations}
This scheme now satisfies $X_n \in \mathcal{C}_M$ for all $n\ge 1$ and any choice of $\epsilon$. Similarly, one can replace~\eqref{eq:update2} by the split-step scheme
\begin{subequations} \label{eq:update2_ss}
    \begin{align}
        X_{n+1/2} &= X_n + \sqrt{C(X_n)}\Xi_n,\label{eq:update2_ss_a}\\
        X_{n+1} &=  m(X_{n+1/2};\epsilon). \label{eq:update2_ss_b}
    \end{align}
\end{subequations}
These split-step schemes have been used in our numerical experiments. We note that (\ref{eq:update_ss}b) can be viewed as a de-noising step applied to a noisy $X_{n+1/2}$. See \cite{MD24} for a recent survey of de-noising techniques used in image processing and their connection to SGM.

\begin{figure}[htbp]
\centering
\includegraphics[width = 0.4\columnwidth]{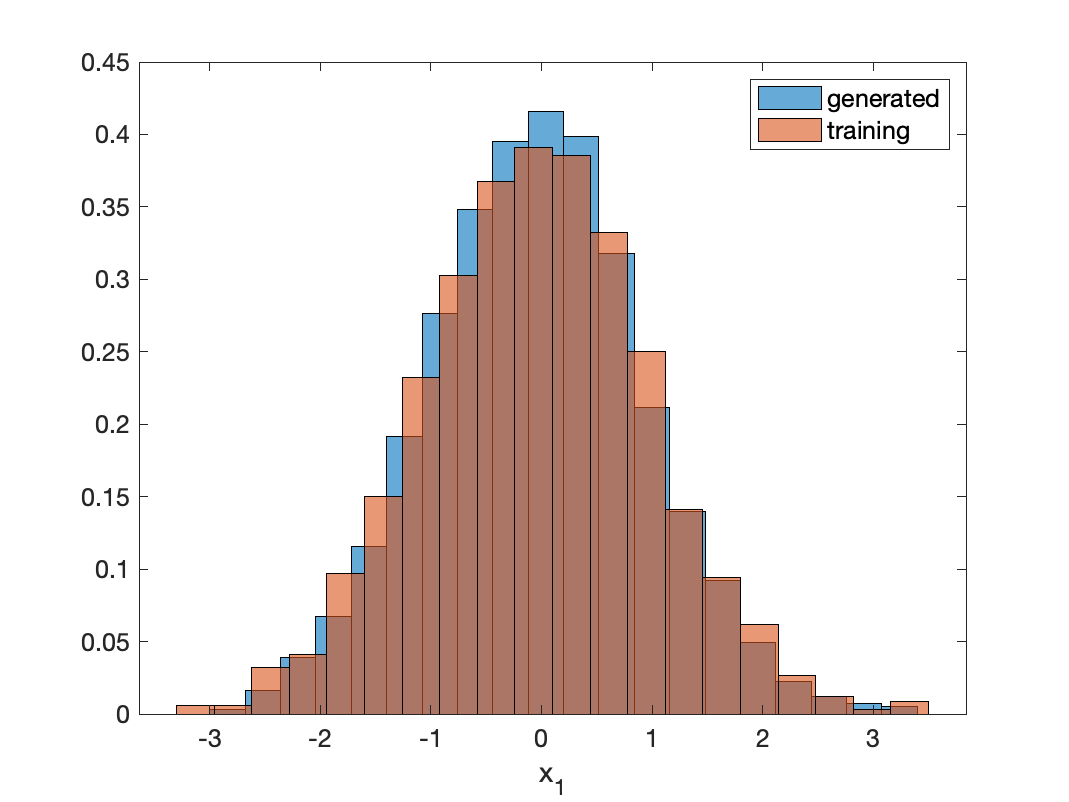} $\qquad$
\includegraphics[width = 0.4\columnwidth]{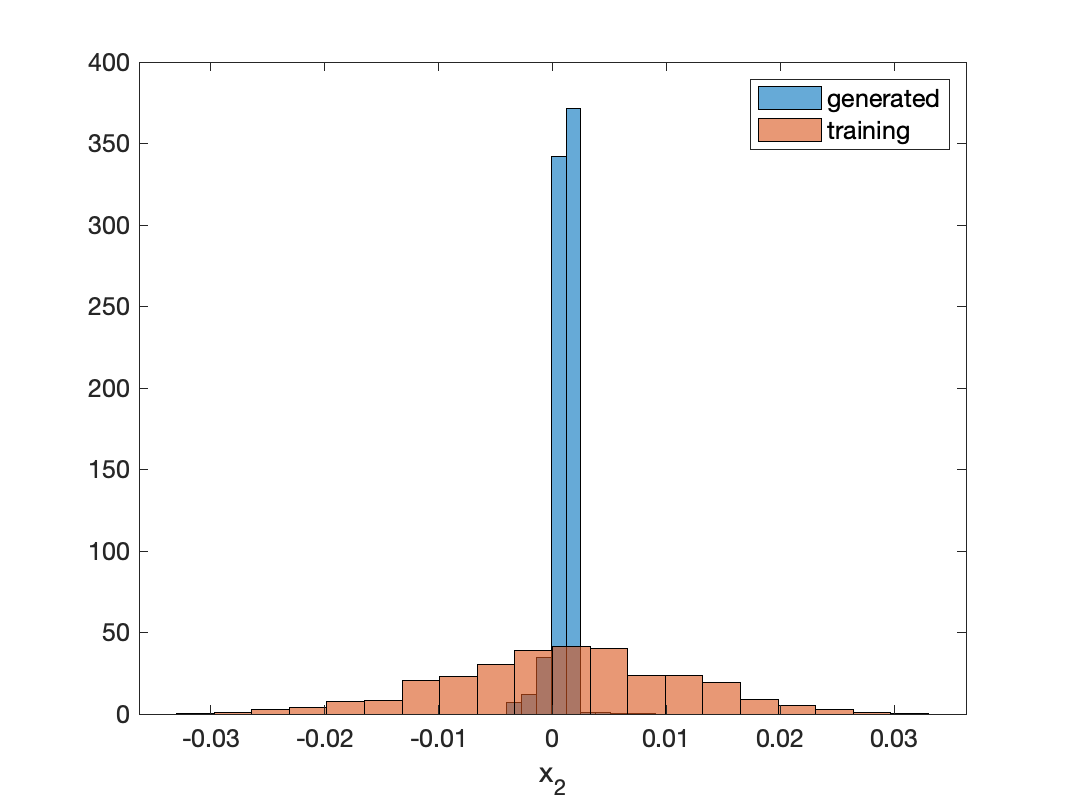}
\caption{Histograms of the $x_1$- and $x_2$-components of the training as well as generated data for Example~\ref{ex:2Db}. One finds that the split-step scheme effectively denoises the $x_2$-component while faithfully reproducing the standard normal distribution in the $x_1$-component.} \label{fig:linear example}
\end{figure}

%\vspace{-1cm}

\begin{example} \label{ex:2Db}
We return to Example~\ref{ex:2D} and demonstrate the performance of the proposed split-step scheme (\ref{eq:update_ss}) for this simple linear problem. We sample $M=10^3$ training samples from the Gaussian distribution with mean zero and covariance matrix 
\begin{equation}
C = \left( \begin{array}{cc} 1+\nu & 0 \\ 0 & 
\nu \end{array} \right)
\end{equation}
with $\nu = 10^{-4}$. The split-step scheme (\ref{eq:update_ss}) is implemented with $\epsilon = 0.1$ and a total of $10^5$ samples are generated. See Figure \ref{fig:linear example} for numerical results. It is found that the proposed scheme samples correctly from the normal distribution in $x_1$ and, at the same time, denoises the $x_2$-component. We recall that a standard Euler--Maruyama discretisation would require step-sizes $\Delta \tau < 0.0002$ and would sample from the noisy distribution ${\rm N}(0,C)$.
\end{example}

\noindent
We finally discuss the computational complexity of the proposed algorithms in terms of the size, $M$, of the training data and the dimension, $d$, of the samples space $\mathbb{R}^d$. In the off-line phase of the algorithm, the accelerated Sinkhorn algorithm of \cite{WR20} converges essentially in order $\mathcal{O}(M^0)$ iterations and each iteration requires matrix-vector multiplications; hence being of order $\mathcal{O}(M^2)$.
Computing $T$ involves calculating $\mathcal{O}(M^2)$ distances between vectors in $\mathbb{R}^d$. The online phase requires the multiplication of the $d\times M$-dimensional data matrix $\mathcal{X}$ with the $M$-dimensional weight vector $p(x)$. The computation of $p(x)$ in turn involves the computation of $M$ distances $\|x-x^{(j)}\|^2$ and inner products of $M$-dimensional vectors. Overall the computational complexity of the online phase is of order $\mathcal{O}(dM)$. 

The more problematic part is the accuracy of the Schr\"odinger bridge approximation to the semi-group 
$\exp(\epsilon \mathcal{L})$. Optimal scaling of $\epsilon$ as a function of $M$ leads to an approximation error of order $\mathcal{O}(M^{-2/(8+d)})$, which requires an exponential increase of the sample size $M$ as a function of the dimension, $d$, of sample space in order to reach a desired accuracy. While improvements in accuracy can be obtained by the variable bandwidth techniques discussed next for low dimensional problems, a localized formulation of the Schr\"odinger bridge sampler, as proposed in \cite{GR24}, can lead to an effectively dimension-independent accuracy. 

%%%%%%%%%%%%%%%%
%
\subsection{Variable bandwidth diffusion}\label{sec:vb}
%
%%%%%%%%%%%%%%%

It is well-known from the literature on diffusion maps that a variable bandwidth can improve the approximation quality for fixed sample size $M$ \cite{DM_variable_bw}. Here we utilize the same idea. However, we no longer insist on approximating the standard generator with $K=I$, since we only wish to sample from the distribution $\pi$ rapidly. Hence, we consider
reversible diffusion processes (\ref{eq:reversible diffusion}) with
\begin{equation} \label{eq:Krho}
K(x) = \rho (x) I.
\end{equation}

It is an active area of research to select a $\rho$ that increases the spectral gap of the associated
generator $\mathcal{L}$, given by
\begin{equation} \label{eq:generator_rho}
\mathcal{L}f = \pi^{-1} \nabla \cdot (\pi \rho \nabla f)
= \nabla \cdot (\rho \nabla f) + \rho \nabla \log \pi\cdot \nabla f ,
\end{equation}
while not increasing computational complexity. A larger spectral gap implies a faster convergence rate \cite{ReyBellet2016ImprovingTC}, indicating that the generated samples are closer to the reference at a finite time, exhibiting a high accuracy. We demonstrate numerically in Section \ref{sec:numerics} that $\rho$ can indeed be used to increase the sampling accuracy. More specifically, the bandwidth $\rho(x)$ is chosen as 
\begin{equation} \label{eq:variable_bandwidth}
\rho(x) = \pi(x)^\beta,
\end{equation}
where $\beta\le 0$ is a parameter and the unknown sampling distribution $\pi$ is approximated by an inexpensive low accuracy density estimator \cite{vb_kde_Salgado-Ugarte, terrell1992variable}. 
One finds that the variable bandwidth parameter $\beta$ in~\eqref{eq:variable_bandwidth} and the
scaling parameter $\epsilon$ both influence the effective step-size in the Markov chain approximation
(\ref{eq:DM_approximation}) for~\eqref{eq:Krho}. 
In order to disentangle the two scaling effects we modify the construction of the entries~\eqref{eq:tij}
in $T_\epsilon$ as follows. We first compute $\pi_i \approx \pi(x^{(i)})$ over all data points and then rescale these typically unnormalized densities:
\begin{equation}
    \tilde{\pi}_i =  Z^{-1}\pi_i, \qquad Z := \frac{1}{M}
    \sum_{j}\pi_j.
\end{equation}
The variable scaling length is then set to
\begin{equation}
    \rho_i = \tilde{\pi}_i^\beta = Z^{-\beta} \pi_i^\beta
\end{equation}
for $i=1,\ldots,M$, i.e., $K(x^{(i)}) = \rho_i I$ in~\eqref{eq:tij} and, more generally,
\begin{equation}\label{vb_K}
K(x) = Z^{-\beta} \pi(x)^\beta.
\end{equation}
The proposed scaling implies that a constant target density $\pi(x)$ leads to $K(x^{(i)}) = I$ in~\eqref{eq:tij} regardless of the chosen $\beta$ value. See Section~\ref{sec:numerics} below for our numerical findings.

In order to derive an appropriate time-stepping scheme, we note that the drift term in the It\^o formulation~\eqref{eq:reversible diffusion} with $K(x) = \rho(x)I$ can be expressed as
\begin{equation} \label{eq:drift_estimator}
\mathcal{L} \,\mbox{id}(x)
= \rho \nabla \log \pi + \nabla \rho
\end{equation}
and, hence, it holds that
\begin{equation}
m(x;\epsilon) = \mathcal{X} p (x) 
\approx \exp(\epsilon \mathcal{L}) \,
\mbox{id}(x) = \mu(x;\epsilon),
\end{equation}
and
\begin{align}
\exp(\epsilon \mathcal{L})\, \mbox{id}(x) \approx x + \epsilon \mathcal{L}\,\mbox{id}(x) = 
x + \epsilon s(x)
\end{align}
as desired (cf. \eqref{eq:score_approximation}). We emphasize that the drift now satisfies $s(x) = \rho(x) \nabla \log \pi(x) + \nabla \rho(x)$ and is no longer equivalent to the score $\nabla \log \pi(x)$. For the numerical variable bandwidth experiments conducted in Section \ref{sec:numerics}, we therefore used (\ref{eq:update_ss}) with $K(x) = \rho(x)I$. We remark that our aim is not to reproduce the actual dynamics of an underlying stochastic process, i.e.~the process associated with the generator \eqref{eq:generator_rho}, but rather to sample from the distribution $\pi(x)$. 

%%%%%%%%%%%%%%%%%%%%%%%%%%%
%
\section{Bayesian inference and conditional sampling}
\label{sec:conditional}
%
%%%%%%%%%%%%%%%%%%%%%%%%%%%

We note that the previously discussed sampling methods can easily be extended to reversible diffusion processes of the form
\begin{equation}
\dot{X}_\tau = - \nabla V(X_\tau) + \nabla \log \pi(X_\tau) + \sqrt{2}\dot{W}_\tau,
\end{equation}
where $V(x)$ denotes a known potential such as the negative log-likelihood function in case of Bayesian inference with $\pi(x)$ taking the role of the prior distribution. We obtain, for example, the adjusted time-stepping scheme
\begin{align}
        X_{n+1/2} &= X_n - \epsilon \nabla V(X_n) + \sqrt{2} \,\Xi_n,\\
        X_{n+1} &= m(X_{n+1/2};\epsilon)
\end{align}
for given samples $\{x^{(i)}\}_{i=1}^M$ from an unknown (prior) distribution $\pi(x)$. This numerical scheme approximately samples from the invariant distribution
$$
\tilde \pi(x) \propto e^{-V(x)} \pi(x).
$$

%%%%%%%%%%%%%%%%%%%%%%%%%%%%%%%%%%%%%%%%%%%
\begin{remark}
We briefly divert to optimization of a regularised minimization problem of the form
\begin{equation}
x^\ast = \arg \min_x \left\{V(x) - \log \pi (x) \right\}
\end{equation}
for given potential $V(x)$ and regulariser $-\log \pi(x)$. Again assuming that only samples $\{x^{(i)}\}_{i=1}^M$ of $\pi(x)$ are available, an algorithm for approximating $x^\ast$ can be defined as follows:
\begin{subequations} 
\begin{align}
        x_{n+1/2} &= x_n - \epsilon \nabla V(x_n),\\
        x_{n+1} &= m(x_{n+1/2};\epsilon).
\end{align}
\end{subequations}
Here $m (x;\epsilon)=\mathcal{X} p (x)$ 
approximates the optimization update associated with the regulariser 
$- \log \pi(x)$. The iteration is stable and any limiting point $x_\infty$ is contained in the convex hull of the data points $\{x^{(i)}\}_{i=1}^M$.
\end{remark}
%%%%%%%%%%%%%%%%%%%%%%%%%%%%%%%%%%%%%%%%%%%%%%%%%%%

\noindent
The sampling schemes~\eqref{eq:update_ss} and \eqref{eq:update2_ss} can furthermore be extended to conditional generative modeling. More specifically, consider a random variable $x = (y,z)$, which we condition on the first component for given $y= y^\ast$. We assume here that the decomposition of $x$ into the parts $z$ and $y$ is given, and we wish to sample from $\pi(z|y^\ast)$ given samples $x^{(i)} = (y^{(i)},z^{(i)})$, $i=1,\ldots,M$, from the joint distribution $\pi(x) = \pi(y,z)$.
 
In order to perform the required conditional sampling, we propose a method which combines approximate Bayesian computation (ABC) with our diffusion map based sampling algorithm. 
As before, we construct vectors of conditional expectation value $m (x;\epsilon)$ based on the samples $\{x^{(i)}\}_{i=1}^M$. We assume that the bandwidth parameter $\epsilon$ used in the diffusion map approximation is also applied in the ABC misfit function, i.e.,
\begin{equation}
L(y,y^\ast) = \frac{1}{2\epsilon}\|y-y^\ast\|^2.
\end{equation}
This suggests the following split-step approximation scheme. Given the last sample $X_n = (Y_n,Z_n)$, we first update the $y$-component using
\begin{equation}
\hat Y_{n} = Y_n - \epsilon \nabla_y L(Y_n,y^\ast) = y^\ast.
\end{equation}
In other words, we replace the current $X_n = (Y_n,Z_n)$ by $\hat X_n = (y^\ast,Z_n)$. Next we apply the
split-step scheme~\eqref{eq:update_ss} to $\hat X_n$, i.e.,
\begin{align}
X_{n+1/2} &= \hat X_n + \sqrt{2 K(\hat X_n)} \Xi_n,\\
X_{n+1} &=  m(X_{n+1/2};\epsilon),
\end{align}
where $m(x;\epsilon) = \mathcal{X} p(x)$ with $\mathcal{X} = (x^{(1)},\ldots,x^{(M)})$, 
and the definition of the probability vectors $p (x)$ follows from~\eqref{eq:probability vectors}. The split-step scheme~\eqref{eq:update2_ss} generalises along the same lines.

We show in Section~\ref{s.stochpara} how conditional sampling with the Schr\"odinger bridge enables drawing inaccessible random variables compatible with a known macroscopic state, a problem known as stochastic subgrid-scale parametrization. We further employ conditional sampling in Section~\ref{sec.L63} to draw random trajectories of a dynamical system where a future state is conditioned on the current state. 

%%%%%%%%%%%%%%%%%
\section{Numerical experiments} 
\label{sec:numerics}
%%%%%%%%%%%%%%%%

%InverseProblem_GaussianOnRing.m
% epsilon=0.009
In this section, we illustrate our method through three numerical examples encompassing different ranges and focal points. In the first two examples, we generate samples using synthetic datasets with increasing dimensions. Our emphasis is on exploring the impact of different diffusions and of employing a variable bandwidth kernel. In the third example, we showcase the proposed conditional generative modeling in Section~\ref{sec:conditional}, applied to a stochastic subgrid-scale parametrization problem. In the final example in Section~\ref{sec.L63}, we employ conditional sampling to generate trajectories of the Lorenz-63 system \cite{Lorenz63} from previously computed training samples.

\subsection{One-dimensional manifold}
To illustrate how well the proposed methods generate statistically reliable samples we consider first the case of $M$ samples $x\in\R^2$. In particular, we consider samples with a polar representation with radius $r=1+\sigma_r \xi_r$ and angle $\theta=\pi/4 + \sigma_\theta \xi_\theta$ with $\sigma_r=0.06$ and $\sigma_\theta=0.6$ and $\xi_{r,\theta}\sim {\rm{N}}(0,1)$. We used $M=2,000$ samples to learn the transition kernel and then generated $10,000$ new samples with an initial condition from the data-sparse tail of the distribution, chosen to be the data point corresponding to the smallest angle. 

We begin by investigating the effect of the two noise models proposed, namely a constant diffusion as in~\eqref{eq:update} with constant bandwidth $K(x)=I$ and the case when the diffusion reproduces the sample covariance $C$ as in~\eqref{eq:update2}. 
%We used $M=2,000$ to learn the semigroup and then generated $10,000$ data with an initial condition at the tail with the data point corresponding to the smallest angle. 
In both cases we use a constant bandwidth $K=I$ in \eqref{eq:tij} and \eqref{eq:tvec} when forming the Schr\"odinger bridge.  
We employ a Langevin sampler with the splitting scheme~\eqref{eq:update_ss} with $\epsilon=0.009$ and~\eqref{eq:update2_ss}, respectively.  Figure~\ref{fig.GaussRing_norho_C_vs_K} shows that choosing the sample covariance as the noise model is clearly advantageous. Whereas both noise models generate samples that reproduce the angular distribution the noise model using a constant diffusion is overdiffusive in the radial direction. In contrast the noise model using the sample covariance nicely reproduces the radial distribution.

We now investigate the effect of a variable bandwidth $K(x)$. We employ the noise model~\eqref{eq:update2_ss} with the sample covariance but use $K(x)$ to determine the diffusion map (cf.~\eqref{eq:tvec}).    
The Langevin sampler~\eqref{eq:update2_ss} is again initialized with the coordinates of the data point corresponding to the smallest angle in the data-sparse tail. Figure~\ref{fig.GaussRing_norho_vs_rho_xy} (left) shows the samples projected onto the convex hull of the data, i.e. outputs of step~\eqref{eq:update2_ss_b}, when a uniform bandwidth $K(x)=I$ with $\epsilon=0.009$ is employed. Although the mean behaviour is well reproduced, it is seen that the generative model fails near the data-sparse tails for large and small angles. Here the value of $\epsilon$ is too small to generate significant diffusion and the samples are aligned on the (linear) convex hull of the widely separated data samples. To mitigate against this behaviour we employ a variable bandwidth $\rho(x)=\pi^\beta$ with $\beta=-1/5$ and a kernel density estimate $\pi(x)$. Figure~\ref{fig.GaussRing_norho_vs_rho_xy} (right) shows how the variable bandwidth kernel better reproduces the distribution in the data-sparse tail regions. Figure~\ref{fig.GaussRing_norho_vs_rho_angle_radius} shows the empirical histograms for the radius and the angle variables of the noisy samples corresponding to Figure~\ref{fig.GaussRing_norho_vs_rho_xy} (i.e. outputs of step~\eqref{eq:update2_ss_a}). Whereas both, the uniform and the variable bandwidth kernels, reproduce the radial distribution very well, the uniform bandwidth fails to reproduce the angular distribution in the tails where the diffusion is not sufficiently strong to allow for efficient mixing and the diffusion process gets stuck in the data sparse region. 

We have seen that a constant uniform bandwidth generates samples which are concentrated in the bulk of the data and which are overly diffusive in the radial direction (cf. Figure~\ref{fig.GaussRing_norho_C_vs_K} (right)). One may wonder if employing a smaller virtual time step $\Delta \tau<\epsilon$ in the Langevin sampler~\eqref{eq:update_tau} will allow a Langevin sampler with constant bandwidth to generate more faithful samples. Figure~\ref{fig.GaussRing_vardt_angle_radius} shows that choosing a smaller time step $\Delta \tau$ in~\eqref{eq:update2_tau}, here with $\Delta \tau=\epsilon/4$ is indeed able to reproduce the radial distribution. However, if the Langevin sampler is initialised with a data point in the center of the data samples, it is not able to diffuse into the tail of the distribution distribution, leading to an under-diffusive empirical histogram for the angles. 

The numerical experiments above suggest that we employ a noise model using the sample covariance $C$ combined with a variable bandwidth $K(x)$ to control eventual data sparse regions. When employing a variable bandwidth our method contains two hyper-parameters which require tuning: the bandwidth factor $\epsilon$ and the exponent $\beta$ in the arbitrary choice of the variable bandwidth $\rho(x)$. Their role will be explored in the following subsection.

\begin{figure}[htbp]
\centering
\includegraphics[width = 0.32\columnwidth]{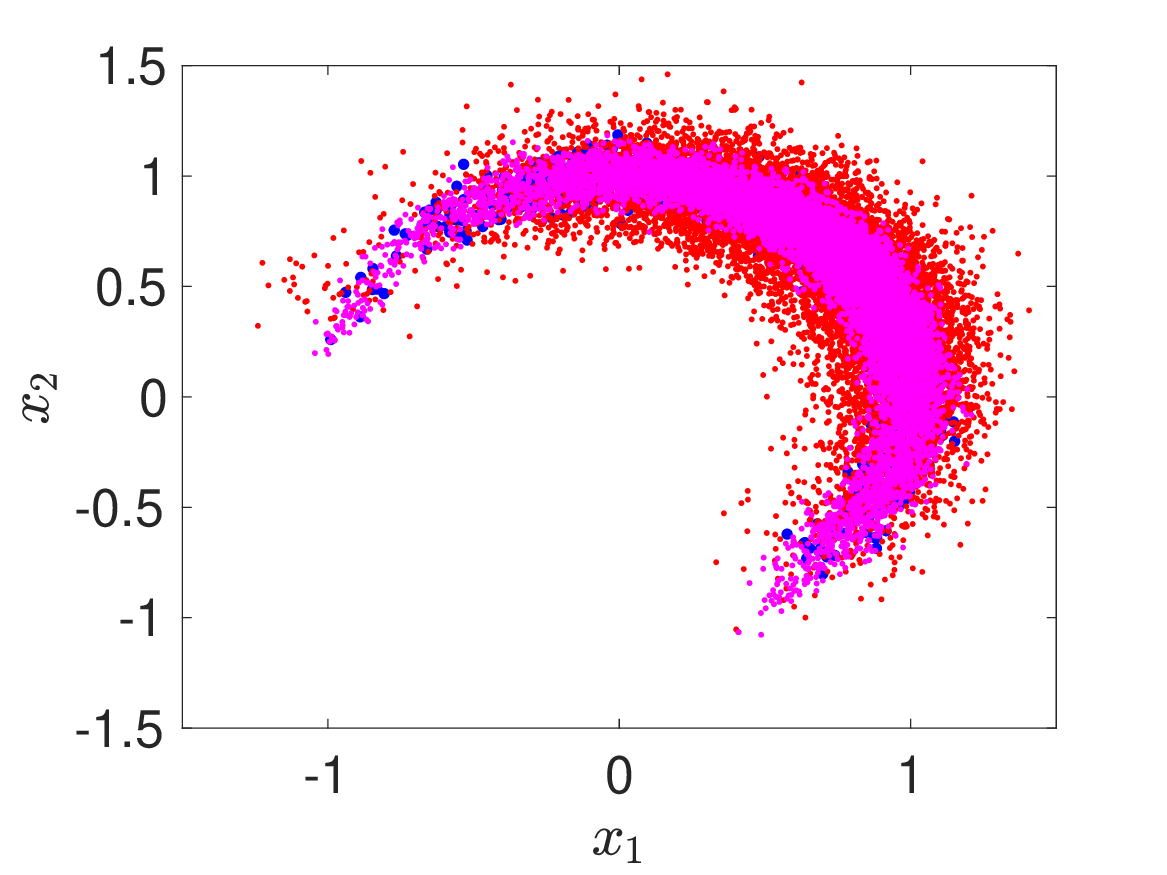}
\includegraphics[width = 0.32\columnwidth]{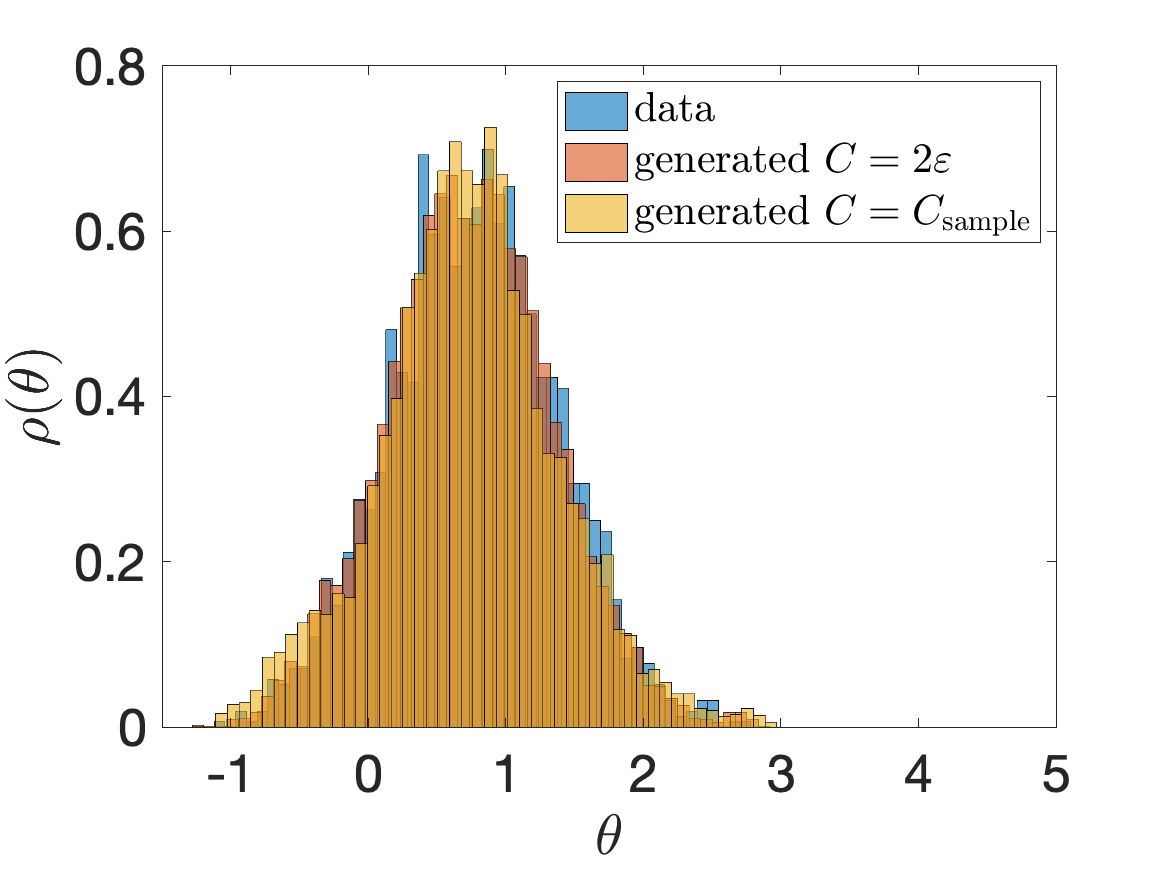}
\includegraphics[width = 0.32\columnwidth]{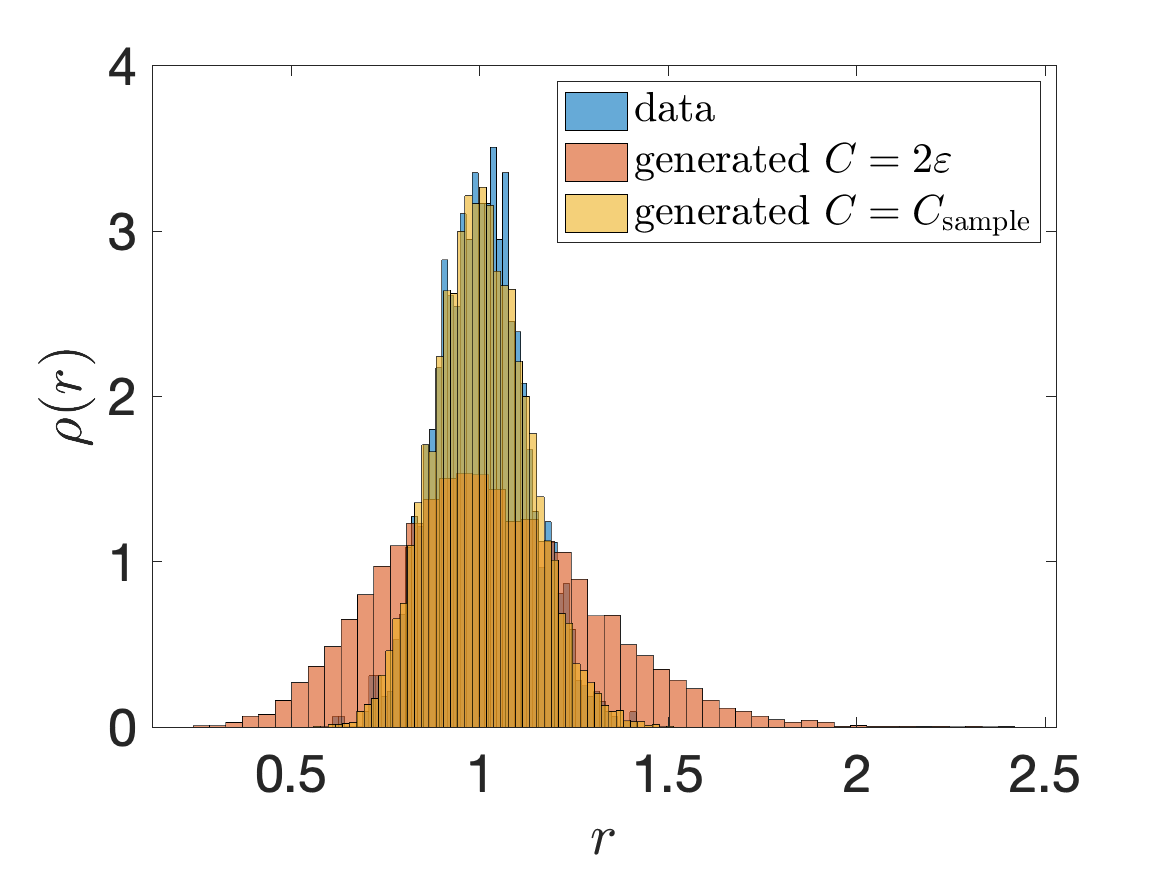}
\caption{Comparison of the different noise models employed by the generative model. We employed a constant bandwidth with $\epsilon=0.009$. Left: Original (blue) and generated data using a constant covariance (red) and the sample covariance $C(x)$ (magenta). Middle: Empirical histograms of the angular variable $\theta$. Right: Empirical histograms of the radial variable $r$.}
\label{fig.GaussRing_norho_C_vs_K}
\end{figure}
%

%\vspace{-1cm}
\begin{figure}[htbp]
\centering
\includegraphics[width = 0.4\columnwidth]{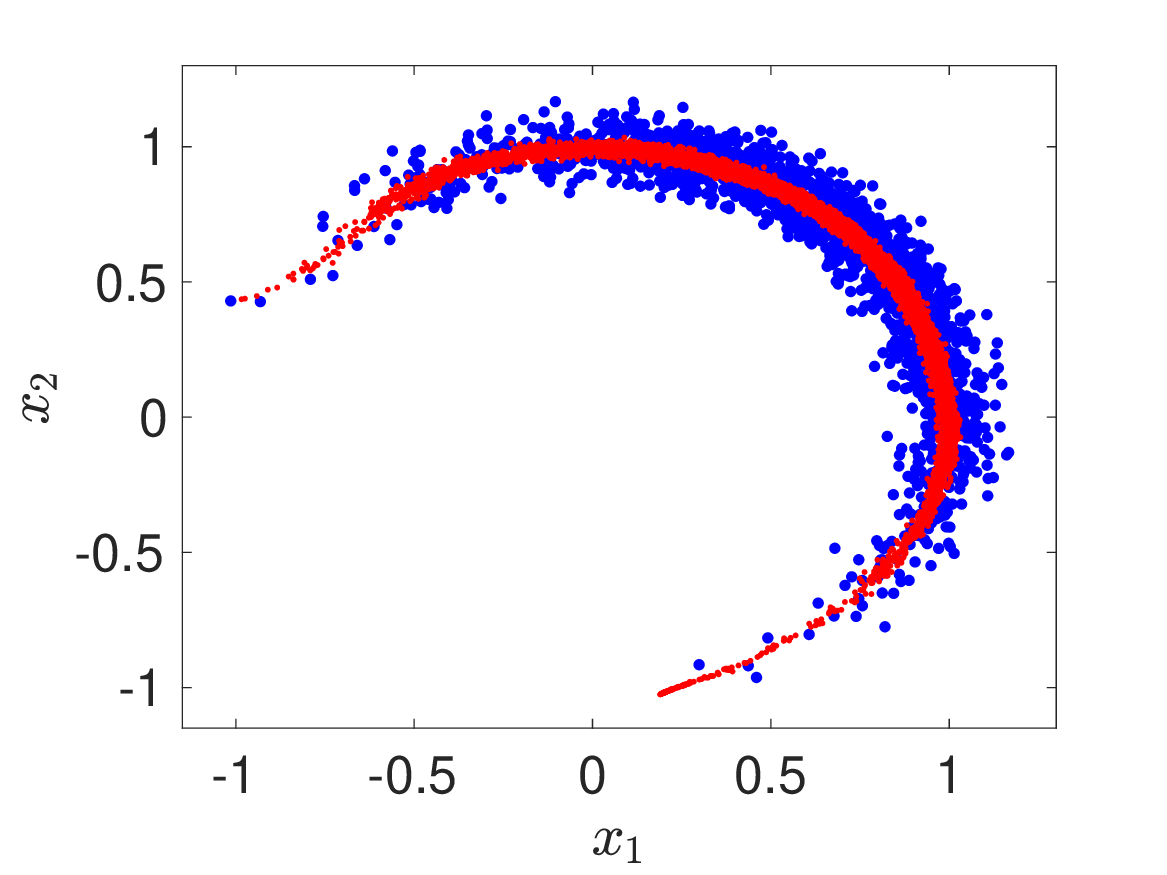} $\qquad$
\includegraphics[width = 0.4\columnwidth]{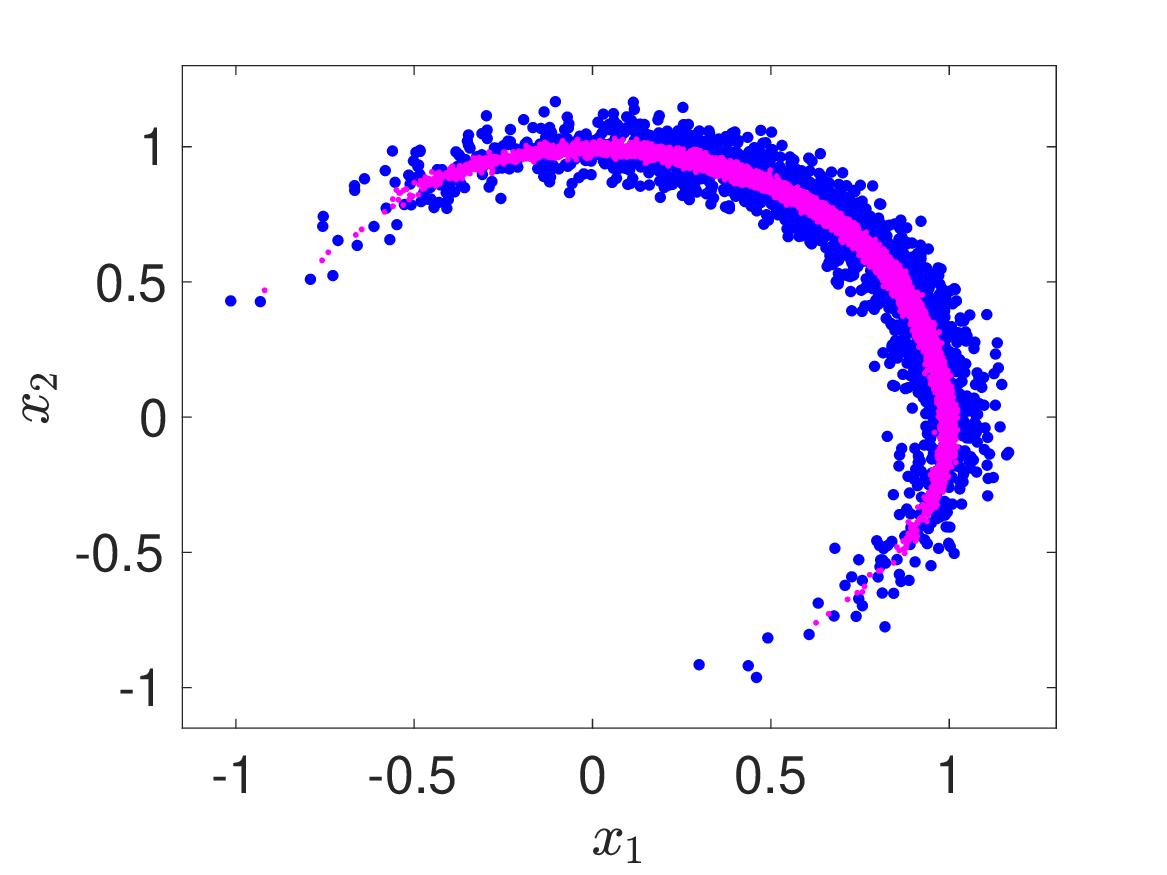}
\caption{Effect of a variable bandwidth $K(x) = \rho(x)I$ in data-sparse regions. For the generative model the Langevin sampler~\eqref{eq:update2_ss} is used and we set $\epsilon=0.009$. Results are shown for the output of step~\eqref{eq:update2_ss_b}. Left: Original (blue) and generated data for a constant bandwidth $K(x) = I$ (red). Right: Original (blue) and generated data for a variable bandwidth $K(x) = \rho(x)I$ with $\rho(x)=\pi(x)^\beta$ with $\beta=-1/5$ (magenta).}
\label{fig.GaussRing_norho_vs_rho_xy}
\end{figure}

\begin{figure}[htbp]
\centering
\includegraphics[width = 0.4\columnwidth]{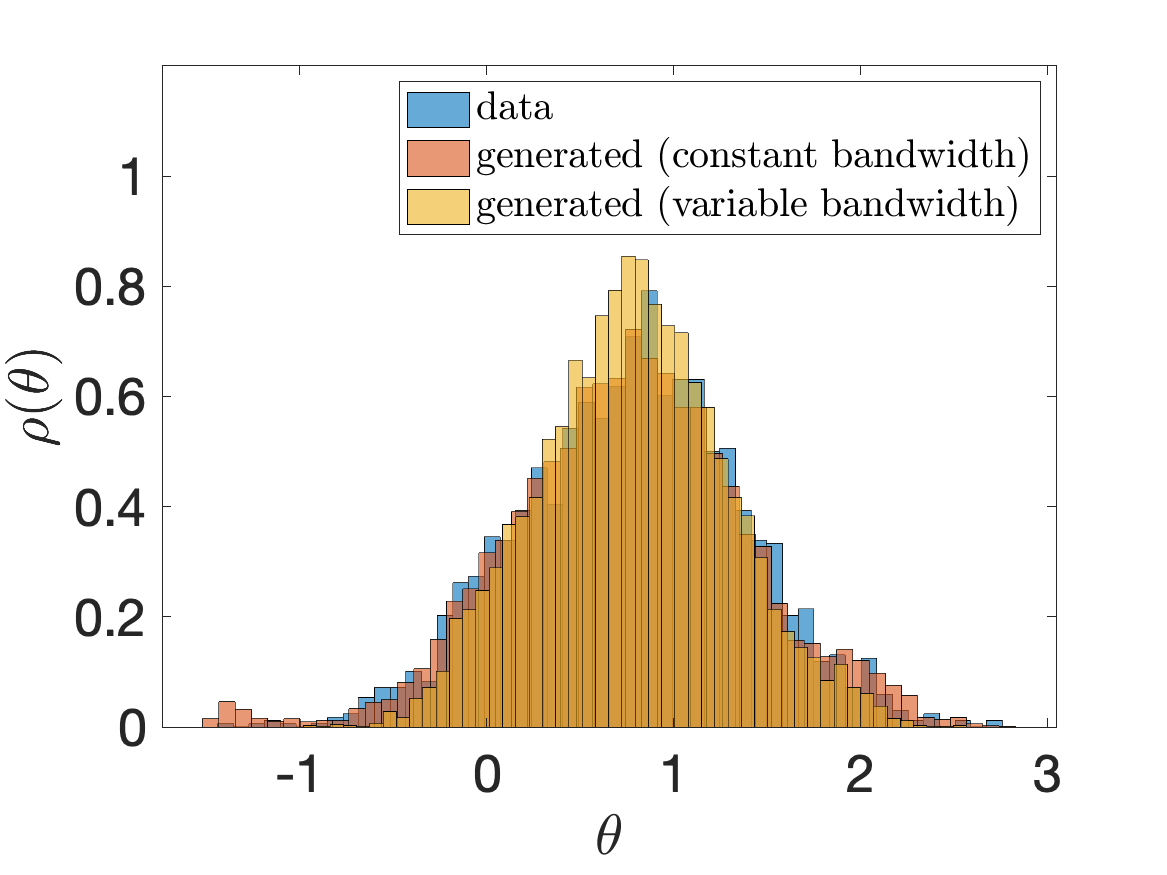} $\qquad$
\includegraphics[width = 0.4\columnwidth]{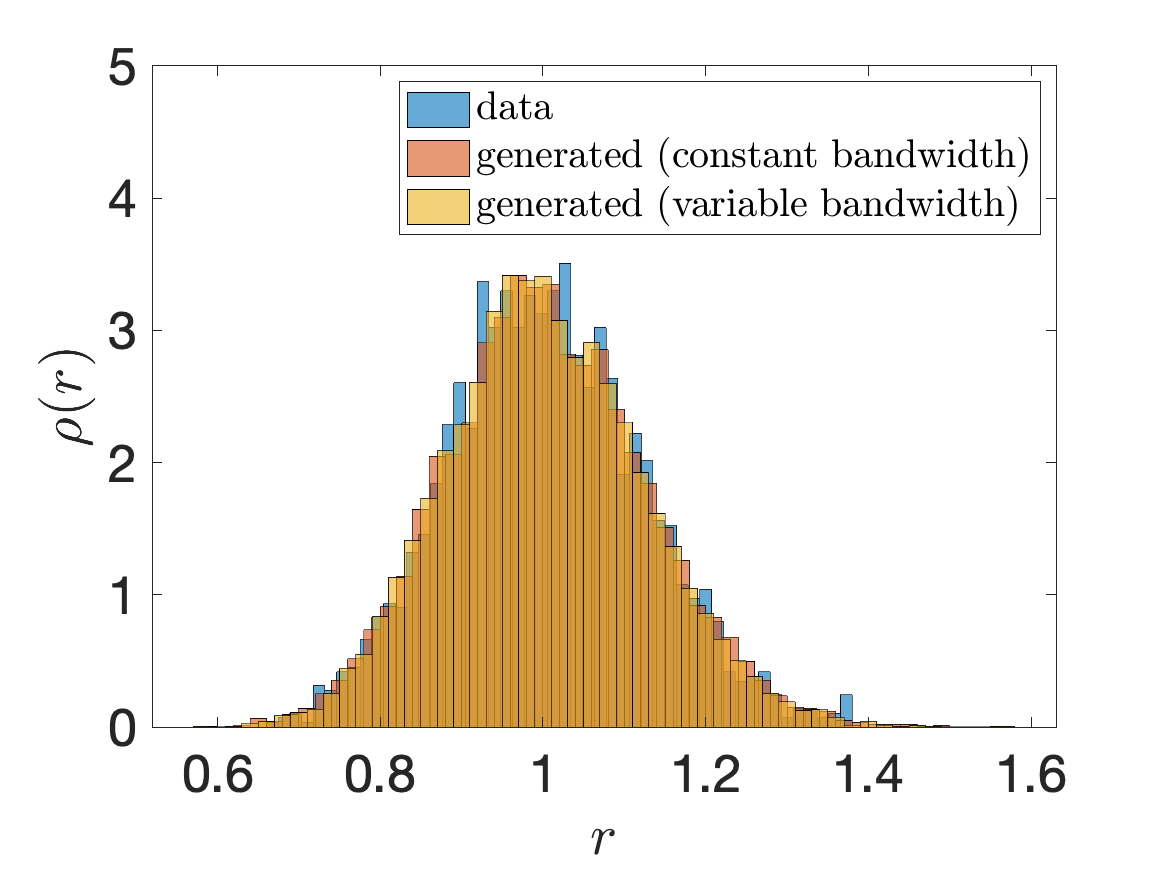}
\caption{Effect of a variable bandwidth $K(x) = \rho(x)I$ on the angular and radial distributions (left and right, respectively). Shown are the original data (blue),  generated data for a constant bandwidth $K(x) = I$ (red) and for a variable bandwidth $K(x) = \rho(x)I$ with $\rho(x)=\pi(x)^\beta$ with $\beta=-1/5$. The data were generated using a constant covariance noise model in~\eqref{eq:update2_ss} and  $\epsilon=0.009$.}
\label{fig.GaussRing_norho_vs_rho_angle_radius}
\end{figure}

\begin{figure}[htbp]
\centering
\includegraphics[width = 0.32\columnwidth]{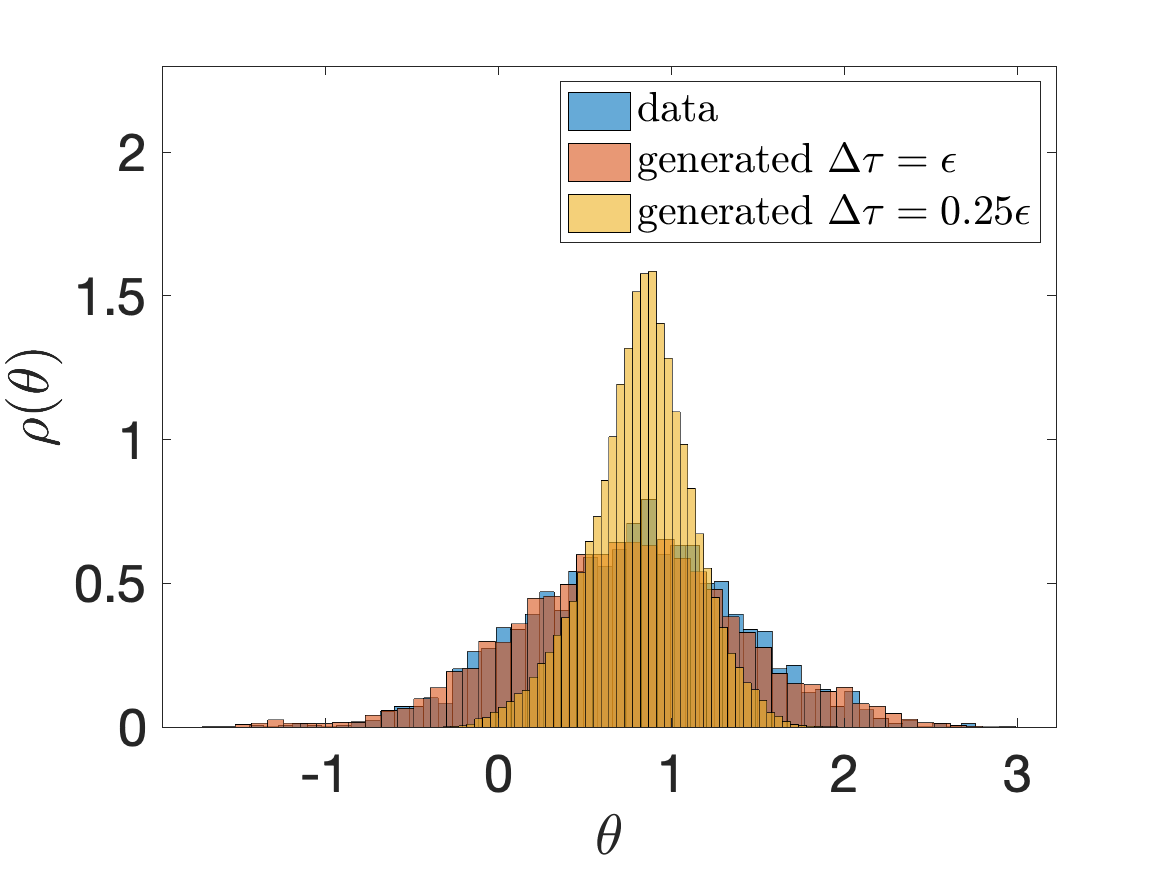}
\includegraphics[width = 0.32\columnwidth]{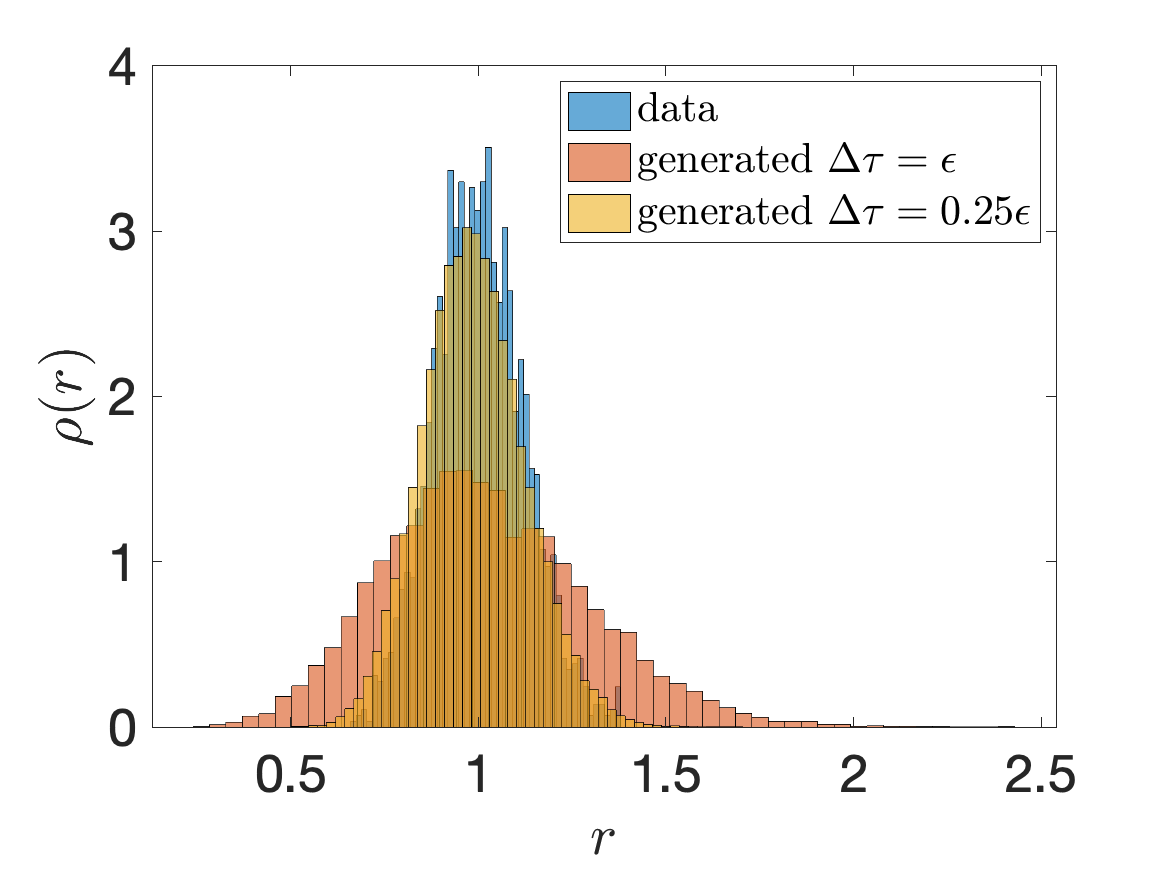}
\includegraphics[width = 0.32\columnwidth]{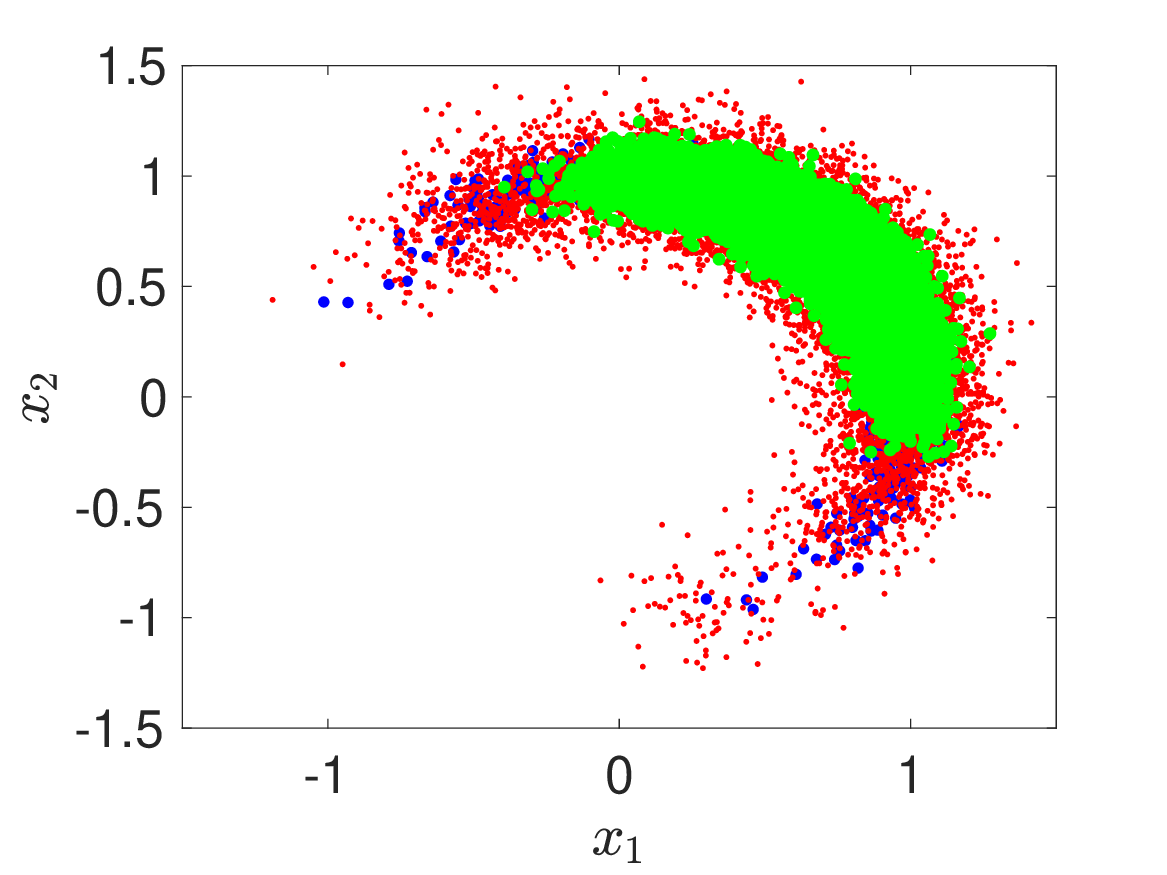}
\caption{Effect of a variable time step $\Delta \tau$ in the Langevin sampler \eqref{eq:update_ss} with constant diffusion $K=1$. Results are shown for the original data, and for $\Delta \tau=\epsilon$ and $\Delta \tau=\epsilon/4$. Throughout a constant bandwidth is used. Left: Empirical histogram of the angular variable $\theta$. Middle:  Empirical histograms of the radial variable $r$. Right: Original (blue) and generated data in the $(x_1,x_2)$-plane with $\Delta \tau=1$ (red) and with $\Delta \tau=\epsilon/4$ (green).}
\label{fig.GaussRing_vardt_angle_radius}
\end{figure}
%
%%%%%%%%%%%%%%%%%%%%%%%%%%%%%%%%%%%%%%%%%%%%%%%%%%%%%%%%%
\subsection{Multi-dimensional manifolds} 
In this numerical example, we show our proposed method on hyper semi-spheres of dimension $d = \{3, 4, 9\}$, using both a fixed bandwidth kernel and a variable bandwidth kernel. Data are generated by firstly sampling $z^{(i)} = (z_1^{(i)},\cdots, z_d^{(i)})$ from a $d$-dimensional standard normal distribution, and then setting $y^{(i)} = (z^{(i)}_1,\cdots,\alpha z^{(i)}_d)$, with $\alpha = 5$ to promote non-uniformity. Finally, the samples $x^{(i)}$ are obtained by normalizing $y^{(i)}$ to achieve the unit length, i.e., $y^{(i)}/\left\Vert y^{(i)}\right\Vert$, and perturbing $y^{(i)}/\left\Vert y^{(i)}\right\Vert$ in the radial direction with $U(0, 0.01)$ noise. An instance of the target samples of three dimensions can be seen in Figure~\ref{fig:samples_3d}. Given a bandwidth $\epsilon$ and a bandwidth function $\rho(x)$, we implement the proposed scheme~\eqref{eq:update2_ss}. For the fixed bandwidth kernel  we set $\rho(x) = 1$. For the variable bandwidth kernel~\eqref{vb_K}, we set $\rho(x) = \pi(x)^\beta$, with $\beta < 0$, and $\pi(x)$ is approximated using a kernel density estimator. We use $M=1,000$ training samples to learn the transition kernel and run a Langevin sampler to generate $50,000$ samples, with the initial data point being $(1,0,\cdots,0)\in \mathbb R^d$. To obtain a better mixing of the Langevin sampler, we take one every $20$ samples of the last $20,000$ generated samples of the chain, resulting in a total of $1,000$ samples. To evaluate the quality of the generated samples, 
%generated with a fixed bandwith kernel and that of a variable bandwidth kernel, 
we compute the regularized optimal transport (OT) distance between the generated samples $\{x_{\mathrm{gen}}^{(i)}\}_{i = 1}^{M_\mathrm{g}}$ and the original reference samples $\{x_{\mathrm{ref}}^{(j)}\}_{j = 1}^{M_\mathrm{r}}$. % from the target distribution. 
The regularized OT distance with entropy penalty $1/\lambda$ is defined as
\begin{align*}
    d_\lambda(x_{\mathrm{gen}}, x_{\mathrm{ref}}) = \min_P \sum_{i,j} P_{ij} C_{ij} - \frac{1}{\lambda} h(P),
\end{align*}
subject to the constraints that 
\begin{align*}
\sum_{j=1}^{M_\mathrm{r}} P_{ij} = M_\mathrm{g}^{-1}, \qquad
\sum_{i=1}^{M_\mathrm{g}} P_{ij} = M_\mathrm{r}^{-1},
\end{align*}
where 
\begin{align*}
    h(P) = -\sum_{i,j} P_{ij} \log P_{ij}
\end{align*}
is the information entropy. The entries of $C\in \mathbb R^{M_g\times M_r}$ are set to be the pairwise Euclidean distances between $\{x_{\mathrm{gen}}^{(i)}\}_{i = 1}^{M_\mathrm{g}}$ and $\{x_{\mathrm{ref}}^{(j)}\}_{j = 1}^{M_\mathrm{r}}$, and each sample is assigned equal weight marginally. 
%The dimension of $P$ and $C$ are both $M_\mathrm{g} \times M_\mathrm{r}$.
We compute the OT distance using the Sinkhorn--Knopp algorithm~\cite{sinkhorn, sinkhorn_knopp}. The number of reference samples is chosen to be $M_\mathrm{r} = 50,000$, and the entropic regularization penalty is set to be $1/\lambda = 100$. We consider the OT distance as a diagnostic to quantify the statistical accuracy of the sampling scheme. 

We then optimize over the parameters $\epsilon$ for a fixed $\lambda$ using grid search. To be more precise, for the Langevin sampler with a fixed bandwidth kernel, we vary $\epsilon$, and compute the OT distance of the generated samples. The best performed $\epsilon$ is chosen to be the one that corresponds to the smallest OT distance, and we call it $\epsilon^*$. For the variable bandwidth kernel, we fix $\epsilon = \epsilon^*$. In order to 
%make fair comparisons, 
disentangle the effect of varying $\epsilon$ and varying $\beta$ 
we use the normalized variable bandwidth as described in~\eqref{vb_K}. We set $\beta=-0.01\times 2^n$ with $n = \{0, \cdots, 8\}$ for $d = \{3, 4, 9\}$. The optimal bandwidth $\epsilon^*$ is reported in Table~\ref{table: hypersphere parameters}
and the comparisons between the fixed bandwidth kernel and the variable bandwidth kernel are presented in Figure~\ref{fig:err_plot}. We observe that by keeping $\epsilon$ fixed, the OT distance becomes smaller for a wide range of $\beta$.

We then examine the quality of generated samples at the optimal $\epsilon$ and $\beta$ along the last coordinate (the nonuniform direction). Similar to the previous study, we compute the one dimensional OT distance of the marginal distribution (see Figure~\ref{fig:err_plot} (right)) and show the histograms and the cumulative density function (CDF) of the samples generated using the fixed bandwidth kernel and using the variable bandwidth kernel in Figure~\ref{fig:cdf_plot}. The benefit of using the variable bandwidth kernel becomes more prominent when focusing on the marginal samples. In the case where we sample a 4-dimensional hyper-semisphere with non-uniformity along the last coordinate, we see in Figure~\ref{fig:cdf_plot} that the empirical CDF of the samples generated with the variable bandwidth kernel closely aligns with the reference (constructed using samples from the target distribution) for the most part. In contrast, the samples generated using the fixed variable bandwidth kernel noticeably diverge from the reference. This aligns with what we observed in Figure~\ref{fig.GaussRing_norho_vs_rho_xy} and Figure~\ref{fig.GaussRing_norho_vs_rho_angle_radius} --- the data generated using the variable bandwidth kernel better resemble the original data. In the cases where the samples are drawn from a 9-dimensional hyper-semisphere, while both methods struggle to generate samples that mirror those from the target distribution, primarily due to the inherent limitations of kernel methods in high-dimensional scenarios, employing a variable bandwidth kernel produces samples that exhibit a closer resemblance to those from the target distribution. 

\begin{figure}[htbp]
\centering
\includegraphics[width = 0.5\columnwidth]{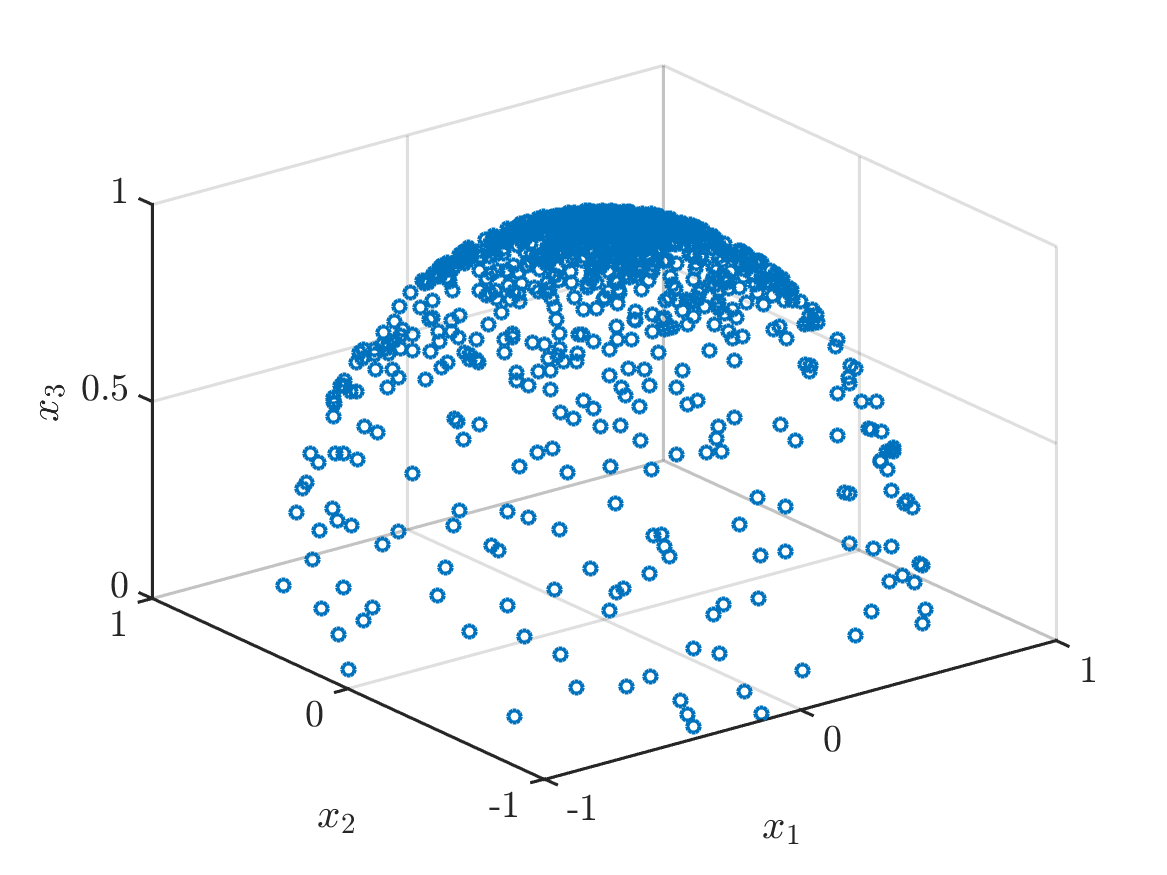}
\caption{Nonuniform samples $x^{(i)}$ on a $2$-dimensional semisphere. The non-uniformity is along the last coordinate $x_3$.}
\label{fig:samples_3d}
\end{figure}

\begin{center}
\begin{table}
\centering
\begin{tabular}{cc }  
\hline
$d$ & optimal bandwidth \\
\hline
3& $\epsilon^* =  0.008$\\
4& $\epsilon^* =  0.010$\\
9& $\epsilon^* =  0.050$\\
\hline
\end{tabular}
\caption{Optimal bandwidth parameters leading to minimal OT distance for different dimensions $d$, obtained using grid search.}
\label{table: hypersphere parameters}
\end{table}
\end{center}

\begin{figure}[htbp]
\centering
\includegraphics[width = 0.4\columnwidth]{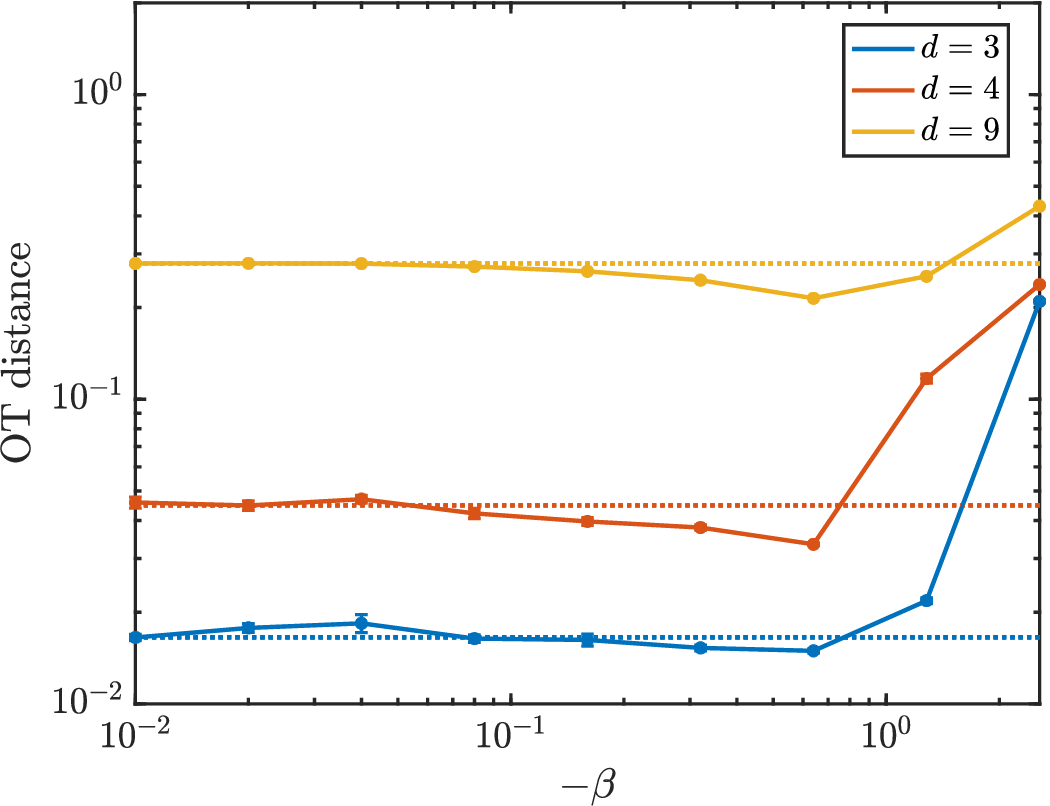} $\qquad$
\includegraphics[width = 0.4\columnwidth]{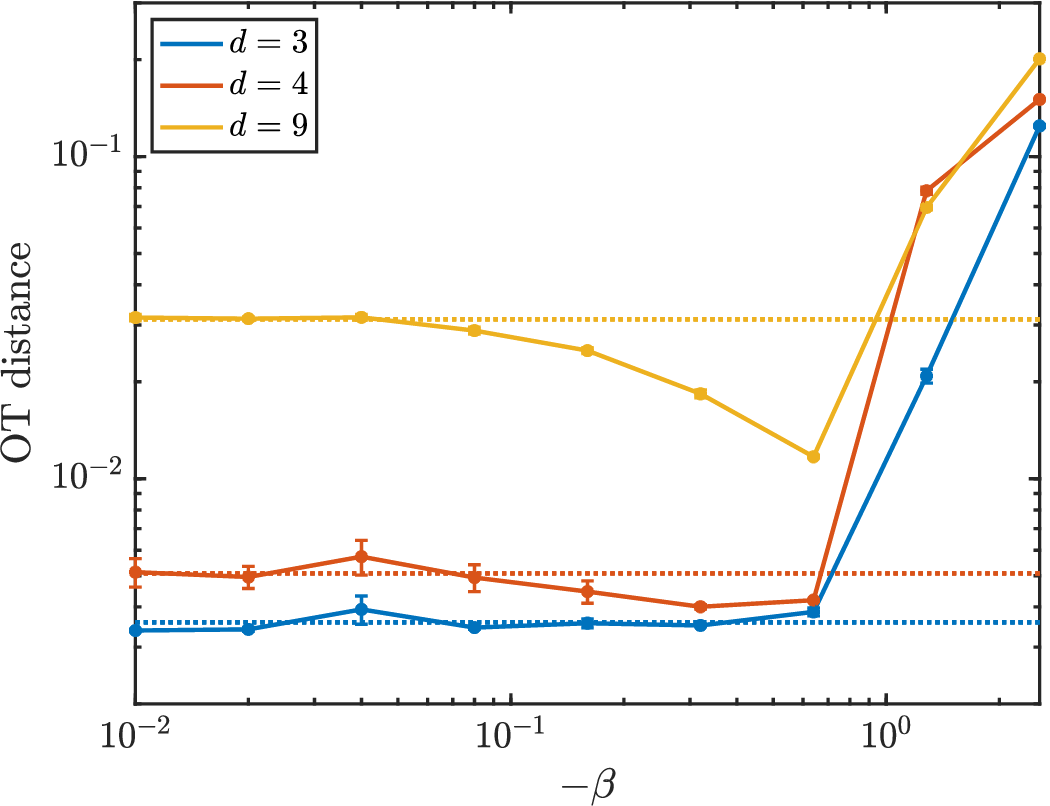}
 \caption{The straight dashed lines denote the OT distances of the samples using a fixed bandwidth kernel, and the solid lines denote the OT distances of the samples using a variable bandwidth kernel. Left:  comparison between the OT distance of samples generated using a fixed bandwidth kernel and  using a variable bandwidth kernel. Right: comparison between the \textit{one-dimensional} marginal OT distance of samples generated using a fixed bandwidth kernel and  using a variable bandwidth kernel, along the last coordinate (non-uniform direction).  Here the Langevin sampler is initialized at $(1,0,\cdots, 0)$ for all cases.}
\label{fig:err_plot}
\end{figure}

%\vspace{-2cm}
\begin{figure}[htbp]
\centering
\includegraphics[width = 0.32\columnwidth]{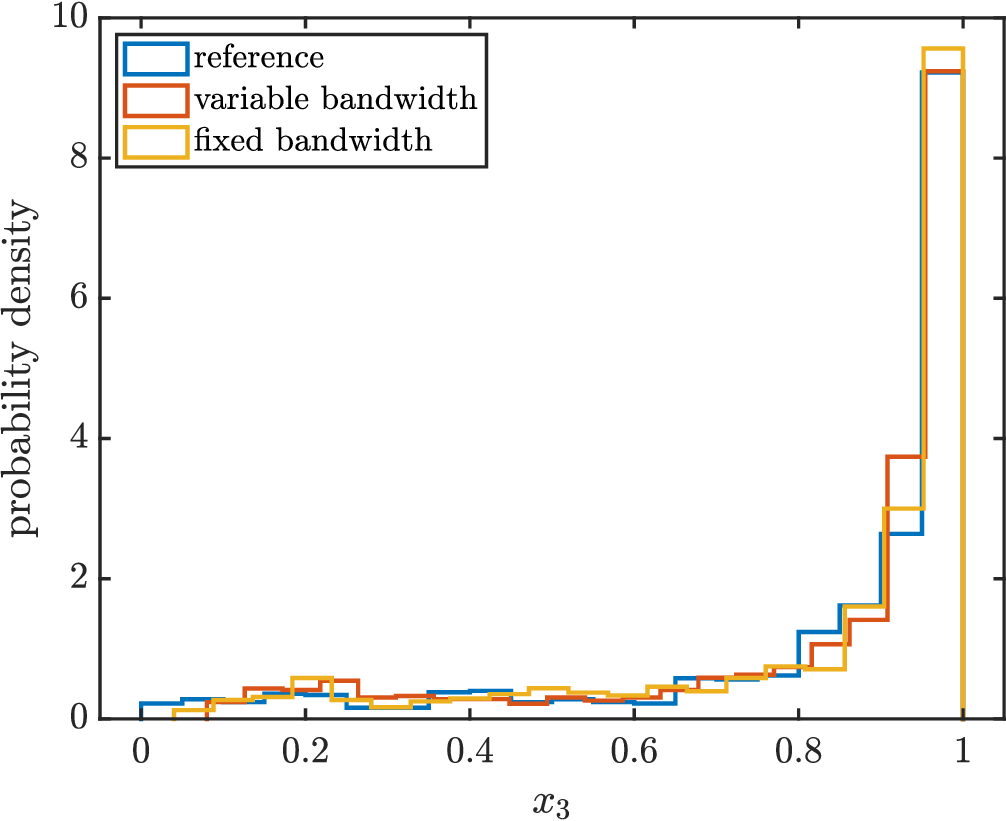}
\includegraphics[width = 0.32\columnwidth]{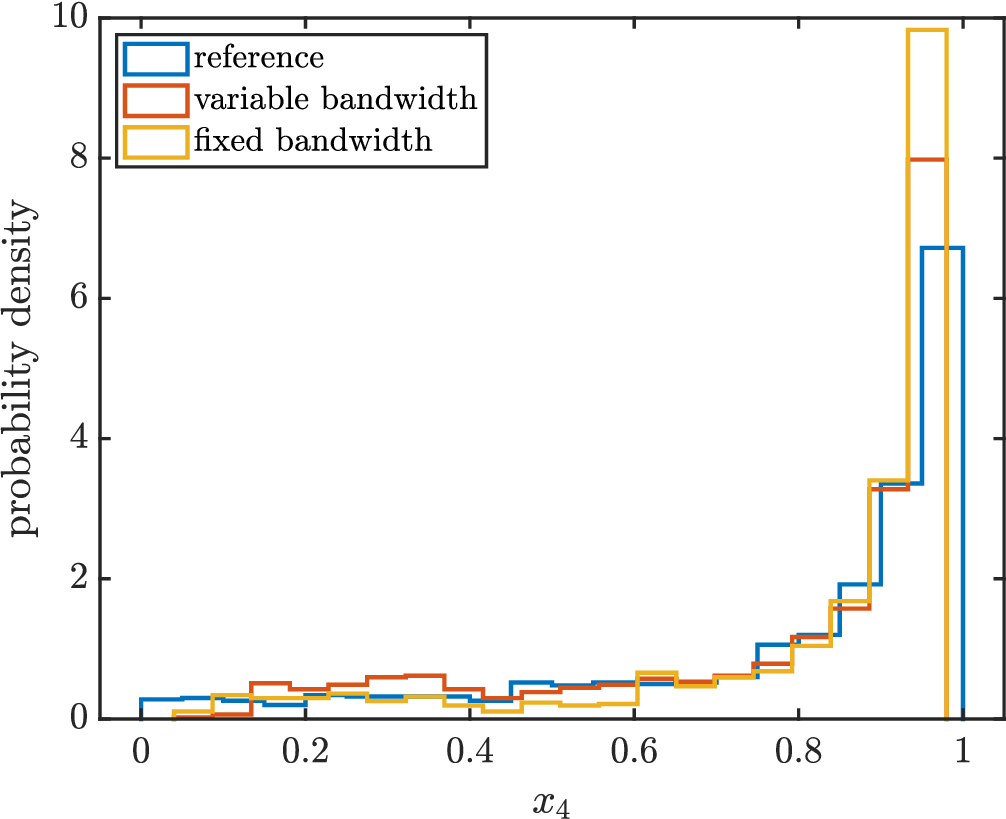}
\includegraphics[width = 0.32\columnwidth]{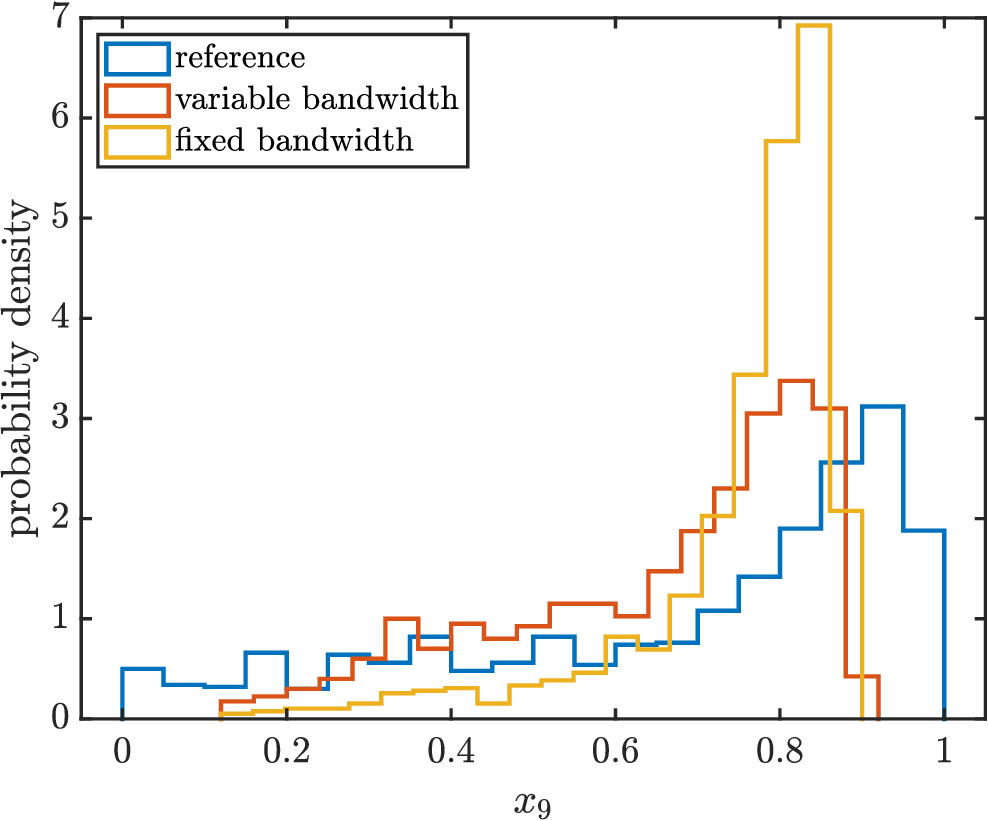}
\includegraphics[width = 0.32\columnwidth]{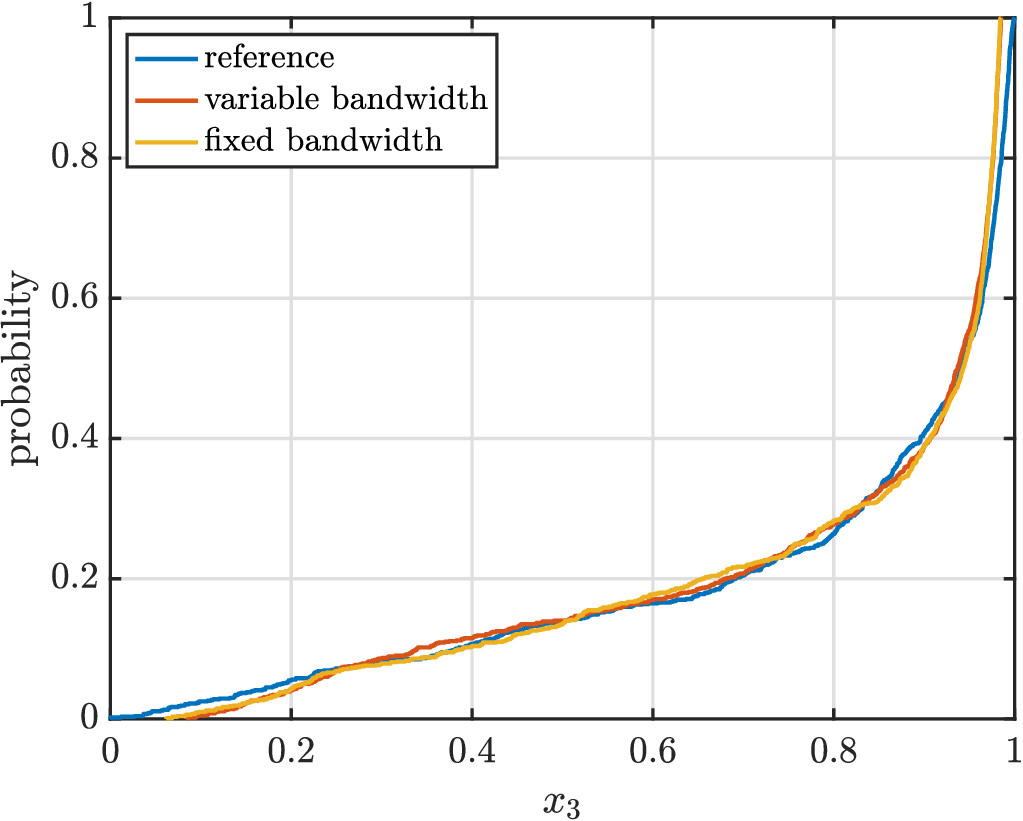}
\includegraphics[width = 0.32\columnwidth]{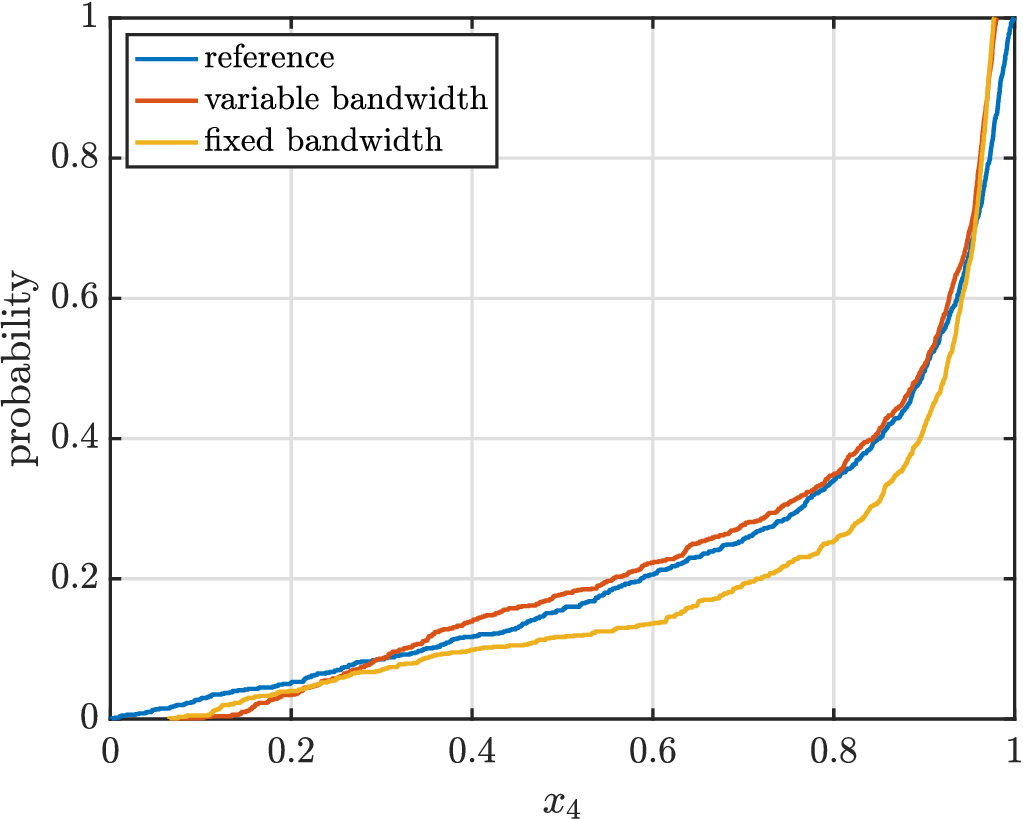}
\includegraphics[width = 0.32\columnwidth]{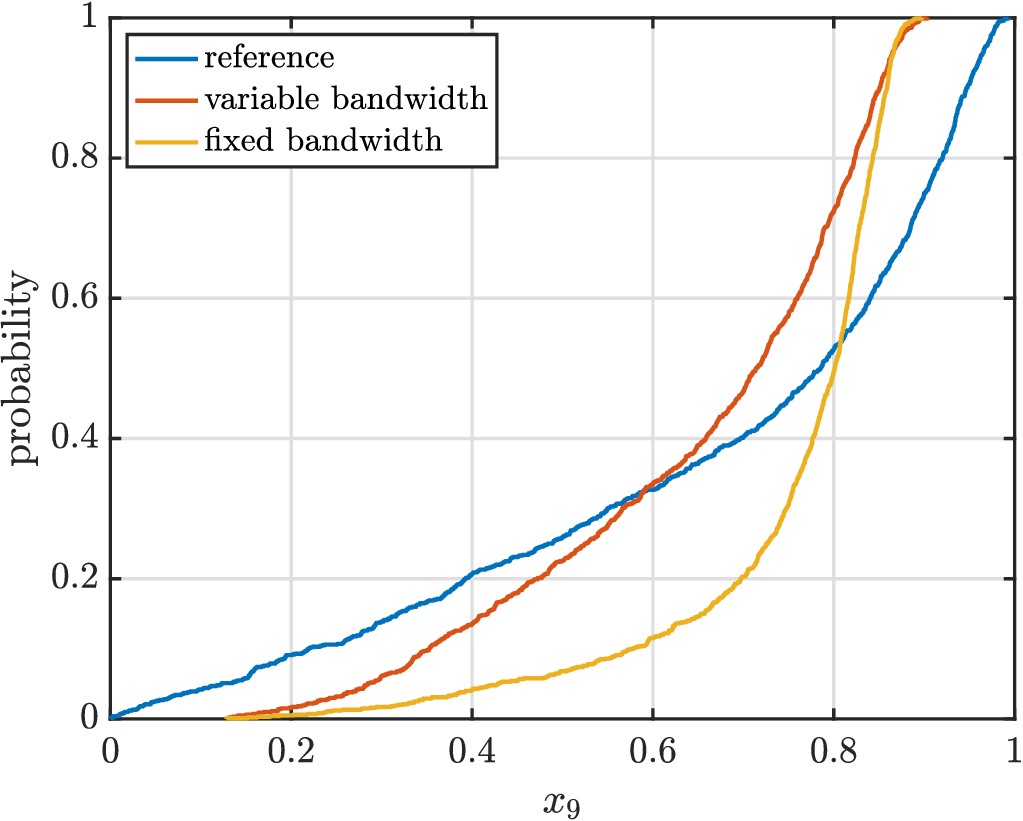}
\caption{Comparisons of empirical histograms (top row) and CDFs (bottom row) of the marginal distribution of the generated samples along the last coordinate. From left to right: the data are sampled from a $\{3,4,9\}$-dimensional (hyper-)semisphere.}
\label{fig:cdf_plot}
\end{figure}

%%%%%%%%%%%%%%%%%%%%%%%%%%%
%
%\section{ABC sampler}
\subsection{Stochastic subgrid-scale parametrization}
\label{s.stochpara}

The conditional sampling algorithm described in Section~\ref{sec:conditional} can be used to perform stochastic subgrid-scale parametrization, a central problem encountered in, for example, the climate sciences. The problem of subgrid-scale parametrization, or more generally of model closure, is the following: given a potentially stiff dynamical system
\begin{subequations}
\begin{align}
\dot z &= F_z(z) + g(z,y;\varepsilon)
\label{e.slow}
\\
\dot y &= F_y(z,y;\varepsilon),
\label{e.fast}
\end{align}
\end{subequations}
where $\varepsilon<1$ denotes the time scale separation between the slow resolved variables of interest $z\in \R^{d_s}$ and the unresolved fast degrees $y\in \R^{d_f}$. Note the notational difference between  the time scale separation parameter $\varepsilon$ and the bandwidth parameter $\epsilon$ used to define the diffusion map. For $\varepsilon \ll 1$ the system is stiff and to ensure numerical stability a small time step $\Delta t<\varepsilon$ is needed. This, together with the potential high-dimensionality of the fast subspace $d_f\gg d_s$, constitutes a computational barrier for simulating the dynamics \eqref{e.slow}--\eqref{e.fast} on the slow time scale of interest. Hence one is interested in obtaining an effective evolution equation for the slow resolved variables $z$ only which captures the essential effect of the unresolved variables $y$. Hence, we seek to determine the effective reduced dynamics
\begin{align*}
\dot z = F_z(z) + \psi(z).
\end{align*}
Here $F_z(z)$ denotes a deterministic drift which we assume to be known {\em{a priori}}, possibly based on physical reasoning. The term $\psi(z)$ denotes the unknown closure term to be learned which may be deterministic or stochastic, and which parametrizes the unknown unresolved fast processes. Deterministic machine learning methods have previously been used to learn $\psi(z)$ as the average effect of the unresolved variables, i.e. the average of $g(z,y;\varepsilon)$ over the (conditional) invariant measure of the fast process \cite{GottwaldReich21a,GottwaldReich21b}.  
Deterministic maps, however, are not able to capture the resolved dynamics with sufficient statistical accuracy, and it is by now well established that the effective equation is often of a stochastic nature \cite{MelbourneStuart11,GottwaldMelbourne13c,KellyMelbourne17,Gottwaldetal17}. 
In the case of infinite time scale separation there are explicit expressions for the effective drift and diffusion term of the effective slow dynamics. However, these terms include integrals over auto-correlation functions and are notoriously hard to estimate. 
Instead we propose to learn the closure term $\psi(z)$ and generate realisations $\psi(z)$ on the fly employing the conditional sampling algorithm described in Section~\ref{sec:conditional}. We consider the situation in which scientists have a good understanding of the resolved dynamics and know the slow vector field $F_z(z)$. Given data of the resolved variables $\{z^{(i)}\}_{i=0}^M$ sampled at equidistant times $\Delta t$, the associated closure term $\psi^{(i)}$ capturing the effect of the unresolved dynamics~\eqref{e.fast} can then be estimated as
\begin{align*}
\psi^{(i)} := z^{(i)} - z^{(i-1)} - F_z(z^{(i-1)}) \, \Delta t , 
\end{align*}
for $i=1,\cdots,M$. This produces the training samples $\{x^{(i)} = (z^{(i-1)},\psi^{(i)})\}_{i=1}^M$ which we collect in the $2d_s\times M$ data matrix $\mathcal{X}$.

The effective dynamics is then provided by the discrete stochastic surrogate model 
\begin{align}
z_{k} = z_{k-1} + F_z(z_{k-1}) \, \Delta t  + \psi_k
\label{e.surr}
\end{align}
for $k\ge 1$ and given $z_0$, where the subgrid-scale terms $\psi_k$ are generated as follows: Given the conditional mean $m(x;\epsilon)$ as described in Section~\ref{sec:diffusion approximation}, we perform the discrete Langevin sampler
\begin{subequations}
\label{eq:condsamp}
\begin{align}
\hat X_n &= (z_{k-1},P_2 X_n)\\
X_{n+1/2} &= \hat X_n + \sqrt{2}\,  \Xi_n\\
X_{n+1} &= m(X_{n+1/2};\epsilon),
\end{align}
\end{subequations}
for $n=1,\cdots,n_s$ with $n_s=100$ and $\Xi_n \sim {\rm N}(0,\epsilon I)$. Here $P_2:\mathbb{R}^{2d_s} \to \mathbb{R}^{d_s}$ is a projector with $P_2 x = \psi$. We finally set $\psi_k = P_2 X_{n_s}$. The assignment of the first component of $\hat X_n$ to $z_{k-1}$ ensures the conditioning of $\psi_k$ on $z_{k-1}$, and $n_s=100$ ensures that the generated samples $\psi_k$ are close to independent from $\psi_{k-1}$. We choose a fixed bandwidth with $\epsilon = 0.001$.

We consider here the particular example with $d_s=1$ and $d_f=3$ given by 
\begin{align}
\dot z &= z(1-z^2) + \frac{4}{90 \varepsilon}h(z) \, y_2, 
\label{e.L63_x}
\end{align}
where the fast dynamics is given by the Lorenz-63 system \cite{Lorenz63}
\begin{subequations}\label{e.L63_y}
\begin{align}
\varepsilon^2 \dot y_1 &= 10(y_2-y_1) \\
\varepsilon^2 \dot y_2 &= 28 y_1 - y_2 -y_1y_3 \\
\varepsilon^2 \dot y_3 &= -\frac{8}{3}y_3 + y_1y_2.
\end{align}
\end{subequations}
We look at the case of effective additive noise with $h(z)=1$ which does not require conditioning on the slow variable $z$, as well as at the case of multiplicative noise with $h(z)=z$. We used {\sc{MATLAB}}'s built-in {\rm{ode45}} routine \cite{MATLAB2022b} to generate the time series with a time-scale separation parameter of $\varepsilon=0.01$. The time series is subsequently sub-sampled with $\Delta t=0.1$. Figure~\ref{fig.L63_additive} shows a comparison between the full multi-scale system~\eqref{e.L63_x}--\eqref{e.L63_y} with additive noise $h(z)=1$ and the stochastic surrogate model~\eqref{e.surr}. The slow $z$-dynamics exhibits stochastic bimodal dynamics. Figure~\ref{fig.L63_additive2} shows a comparison for the multiplicative case $h(z)=z$ which yields unimodal slow dynamics. It is clearly seen that the surrogate model~\eqref{e.surr} obtained by the generative conditional sampler generates statistically reliable dynamics. 

We remark that the above proposed method to draw random samples from the closure term does not require the knowledge of the time-scale separation parameter $\varepsilon$. The closure term $\psi$ only depends on the time series $\{z^{(i)}\}_{i=0}^M$ of the resolved variables and the knowledge of the slow vector field $F_z$ (which may be zero). 

In the next subsection we show that one can indeed draw full sample trajectories of a dynamical system even without any time-scale separation using the conditional Schr\"odinger bridge sampler.

% load('bimodal_L63_eps0p01_mult_multiscale_ABC_N120K.mat')
%
\begin{figure}[htbp]
\centering
\includegraphics[width = 0.32\columnwidth]{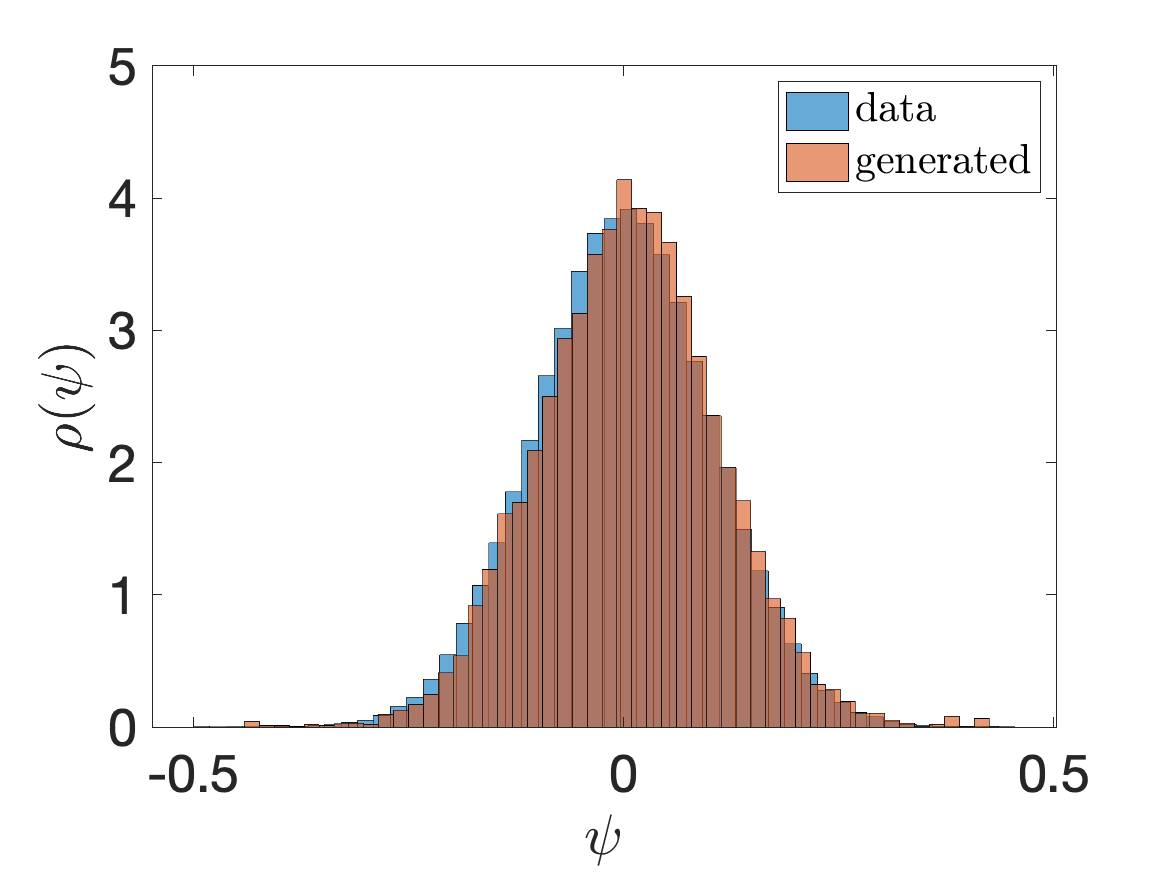}
\includegraphics[width = 0.32\columnwidth]{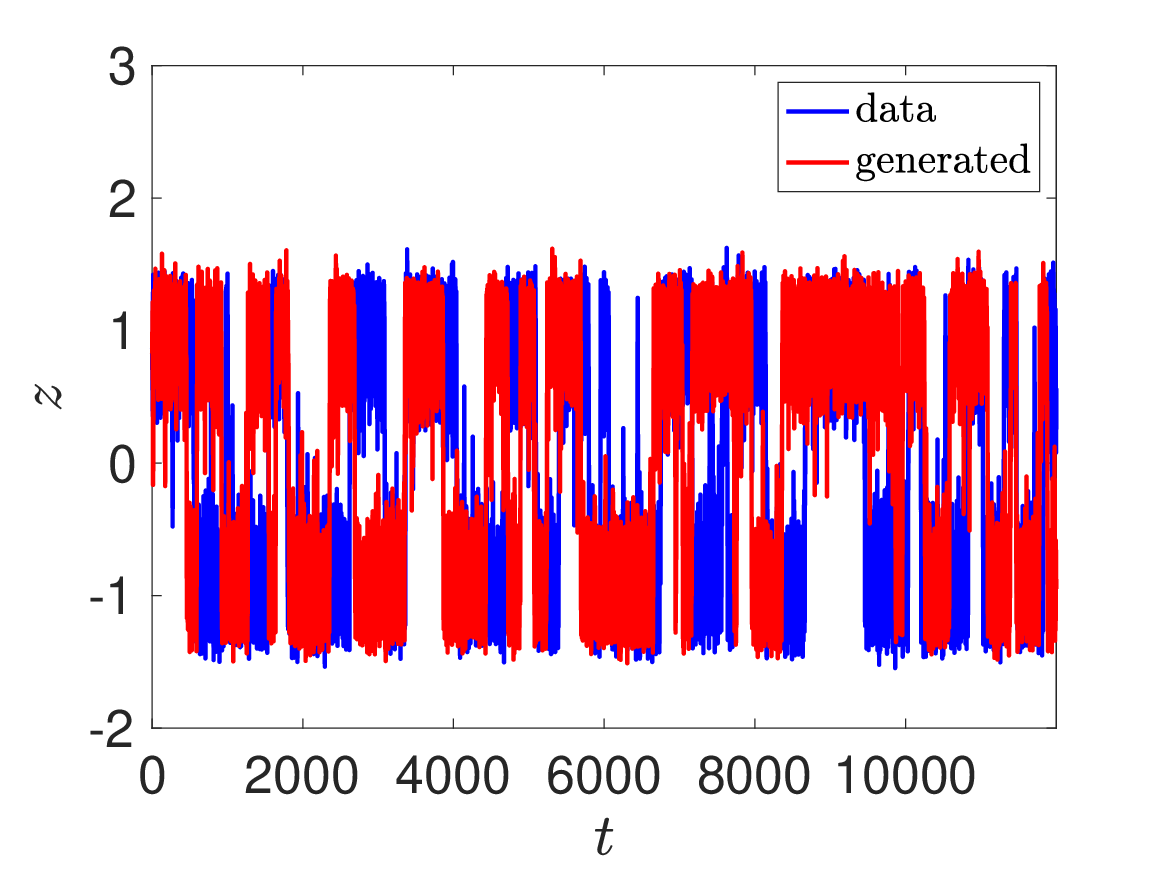}
\includegraphics[width = 0.32\columnwidth]{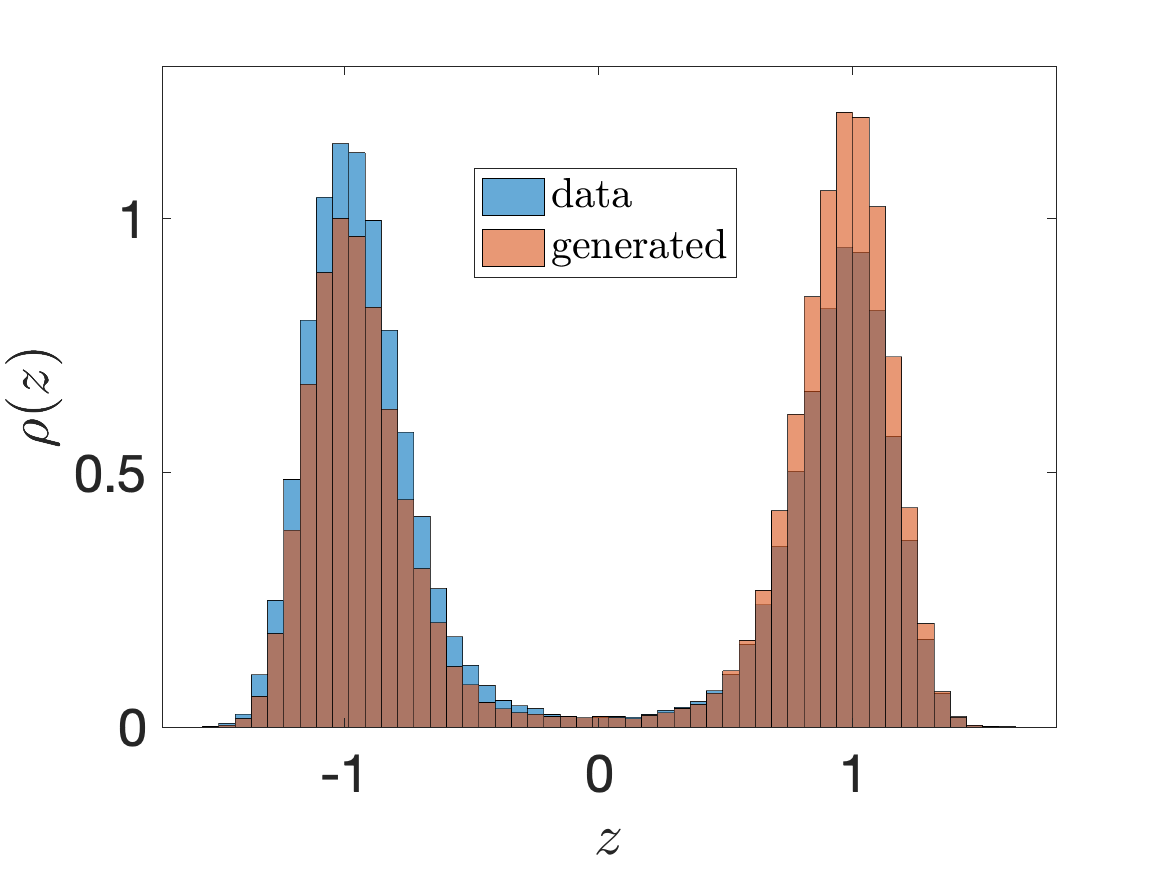}
\caption{Results for the stochastic subgrid-scale parametrization for the multi-scale system~\eqref{e.L63_x}--\eqref{e.L63_y} with $\varepsilon=0.01$ and with additive noise $h(z)=1$. Shown are results obtained by integrating the full multi-scale system and by the stochastic subgrid-scale parametrization scheme using our generative sampler \eqref{eq:condsamp} trained with $M=120,000$. Left: Empirical histograms of the closure term $\psi=\psi(z)$. Middle: Time series of the slow variable $z(t)$. Right: Empirical histograms of the slow variable $z$.}
\label{fig.L63_additive}
\end{figure}
%

% load('bimodal_L63_eps0p01_mult_multiscale_ABC.mat')
%
\begin{figure}[htbp]
\centering
\includegraphics[width = 0.32\columnwidth]{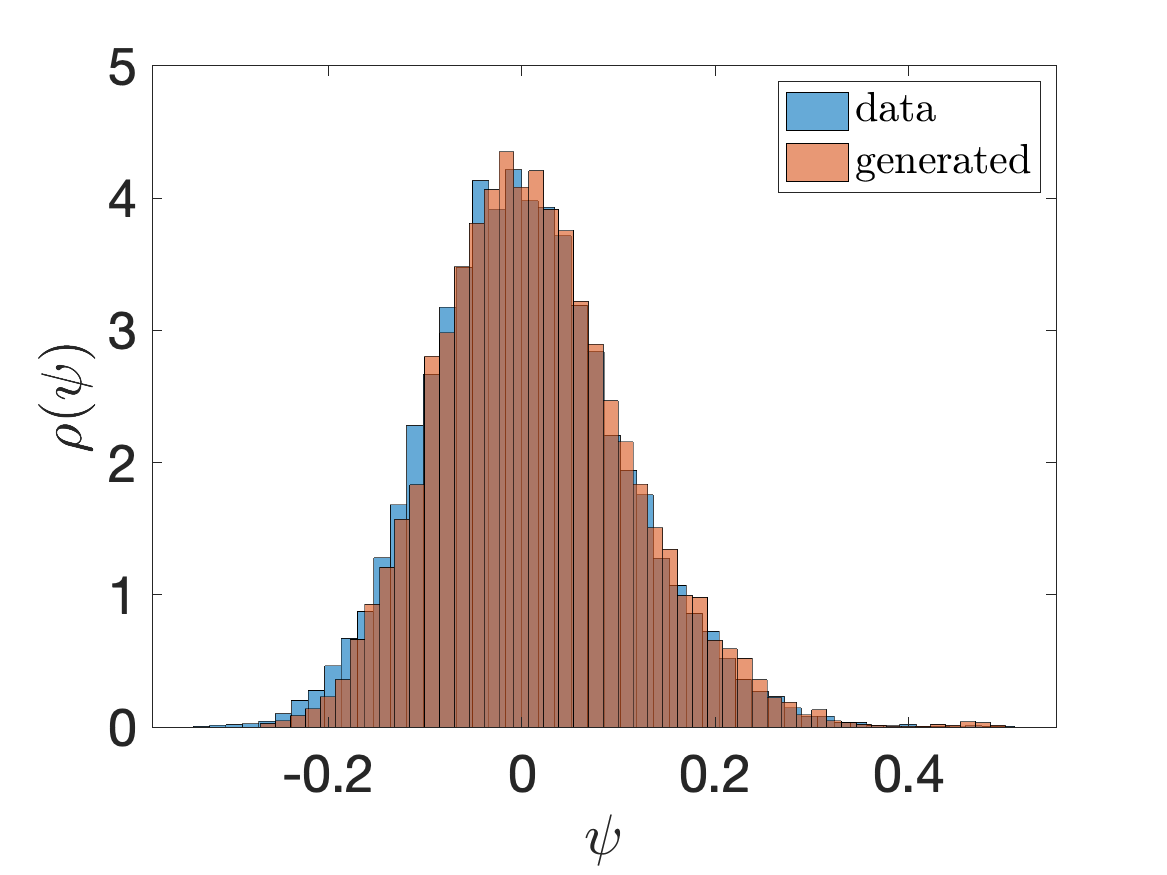}
\includegraphics[width = 0.32\columnwidth]{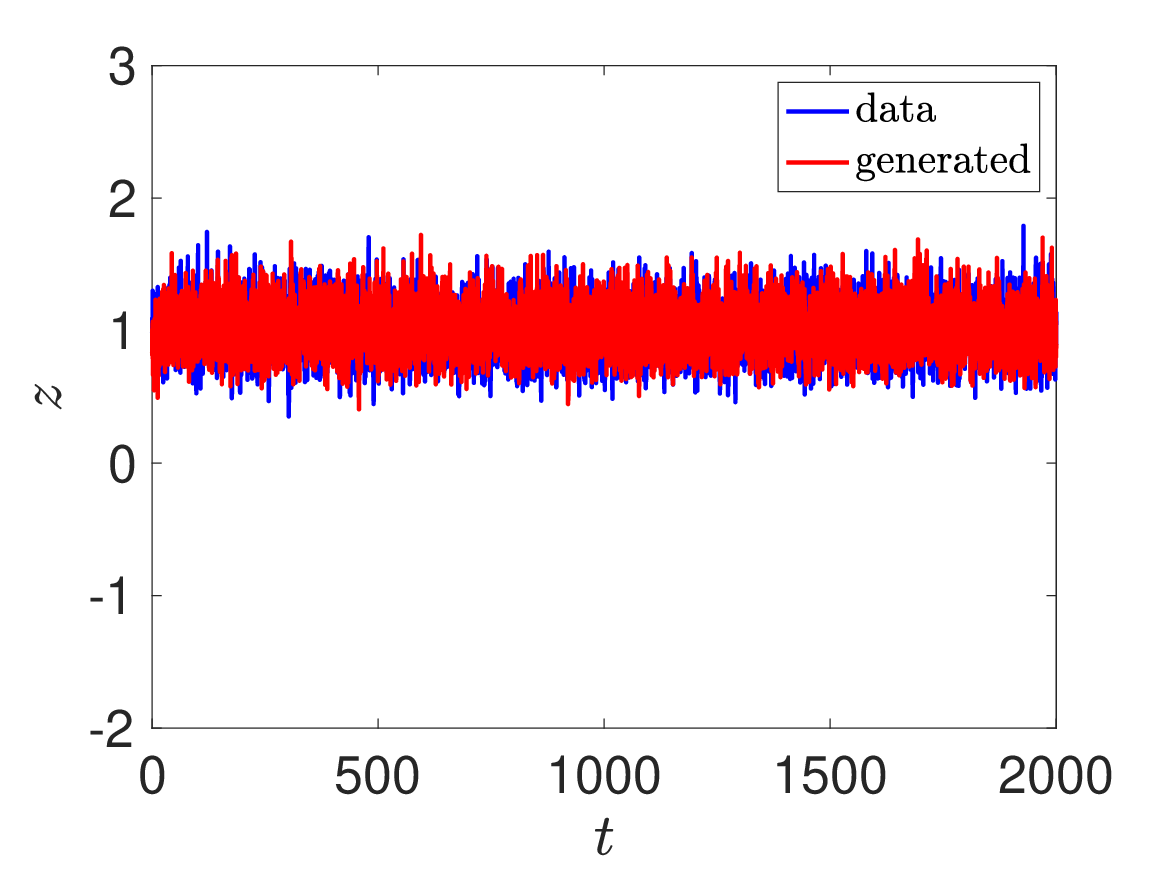}
\includegraphics[width = 0.32\columnwidth]{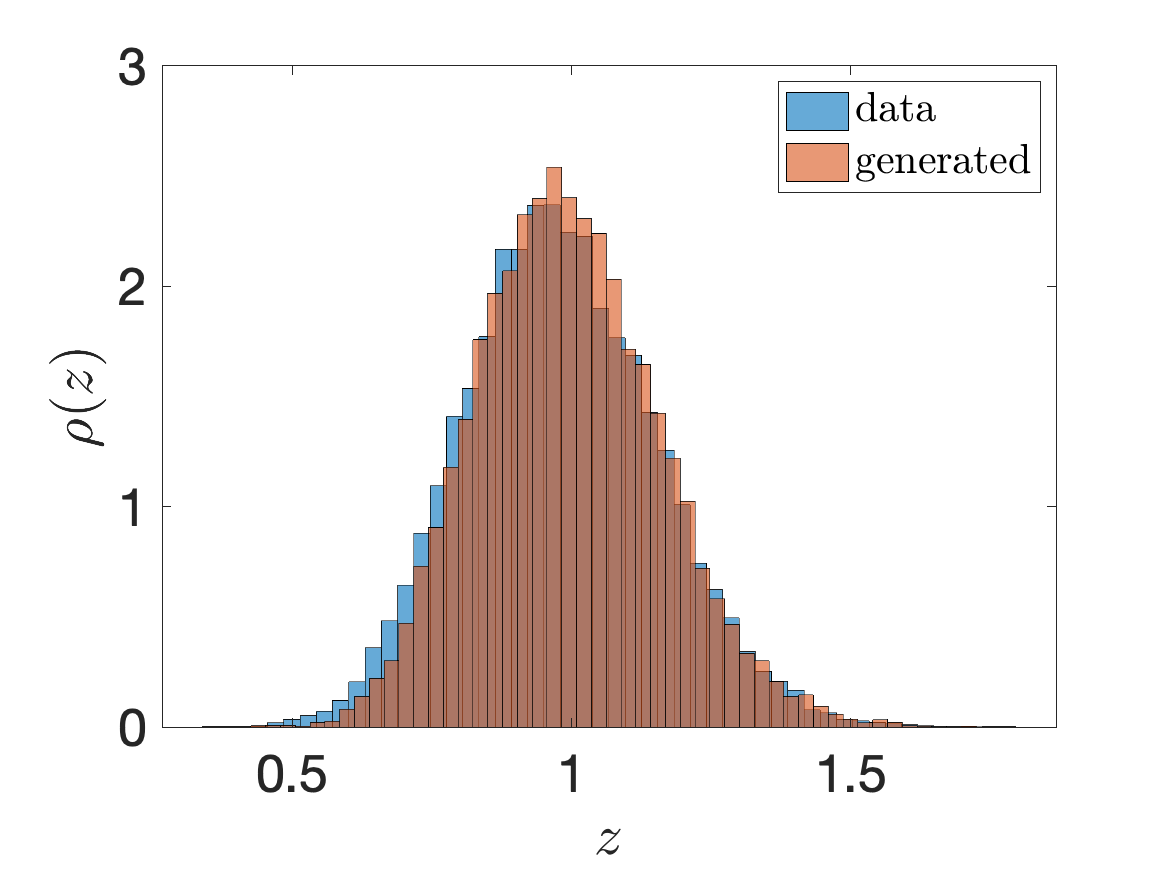}
\caption{Results for the stochastic subgrid-scale parametrization for the multi-scale system~\eqref{e.L63_x}--\eqref{e.L63_y} with $\varepsilon=0.01$ and  with multiplicative noise $h(z)=z$. Shown are results obtained by integrating the full multi-scale system and by the stochastic subgrid-scale parametrization scheme using our generative sampler \eqref{eq:condsamp}, trained with $M=20,000$. Left: Empirical histograms of the closure term $\psi=\psi(z)$. Middle: Time series of the slow variable $z(t)$. Right: Empirical histograms of the slow variable $z$.}
\label{fig.L63_additive2}
\end{figure}

%%%%%%%%%%%%%%%%%%%%%%%%%%%
%
\subsection{Generative modelling of dynamical systems}
\label{sec.L63}
We now employ the Schr\"odinger bridge sampler to perform sequential conditional sampling and generate typical trajectories from a dynamical system. We consider the Lorenz-63 system \cite{Lorenz63} for $y=(y_1,y_2,y_3)$ 
given by (\ref{e.L63_y}) with $\varepsilon = 1$. Given a long time trajectory $\{y^{(i)}\}_{i=0}^M$ with equidistantly sampled data points $y^{(i)}=y(t_i)$ with $t_{i+1}-t_i=\Delta t$ for all $i$, we construct the Schr\"odinger bridge as a coupling for the joint probability function $\pi(y(t),y(t+\Delta t))$ from the $M$ training pairs $x^{(i)}=(y^{(i-1)},y^{(i)})\in \mathbb{R}^6$. Our aim is to generate new trajectories $\{y_k\}_{k\ge 0}$ in time intervals of $\Delta t$ for given $y_0 = y(t_0)$ and $y_1=y(t_1)$ sequentially, similarly to the generative modelling of texts: Given $x_{k-1} = (y_{k-1},y_{k})$, generate a new $x_{k} = (y_{k},y_{k+1})$, which is not part of the initial training samples. Obviously the first component of $x_{k}$ needs to be conditioned on the second component of $x_{k-1}$. As in Section~\ref{s.stochpara}, to decorrelate we perform $n_s=20$ Langevin sampling steps 
\begin{subequations}
\label{eq.condsamp}
\begin{align}
\hat X_{n} &= (y_k,P_2 X_{n})\\
X_{n+1/2} &= \hat X_n + \sqrt{2}\,  \Xi_n\\
X_{n+1} &= m(X_{n+1/2};\epsilon),
\end{align}
\end{subequations}
for given $y_k$ with $\Xi_n \sim {\rm N}(0,\epsilon I)$. Here $P_2:\mathbb{R}^6 \to \mathbb{R}^3$ denotes the project onto the second three components of $x \in \mathbb{R}^6$. We finally set $y_{k+1} = P_2 X_{n_s}$.

In Figure~\ref{fig.L63_gen} we show that the conditional sampling produces trajectories which resemble those of the actual Lorenz-63 system, as well as having the same asymptotic statistical behaviour as seen by the reconstruction of the famous butterfly attractor. We choose a fixed bandwidth with $\epsilon = 0.05$, and used $10,000$ training data $x(t_n)$ with $\Delta t = 0.1$.

The proposed approach can be compared to a direct approximation of the time-$\Delta t$-propagator
\begin{equation}
y_{n+1} = \Psi(y_n)
\end{equation}
via, for example, a random feature map approximation as proposed in \cite{GottwaldReich21a}. More specifically, using a smaller observation interval of $\Delta t = 0.02$ and noisy data, highly accurate predictions were achieved in \cite{GottwaldReich21a} using $D_r = 300$ randomly chosen feature maps. The output weights $W \in \mathbb{R}^{3\times D_r}$ were learned using the ensemble Kalman filter \cite{reichcotter15}. In comparison, the proposed Schr\"odinger bridge sampler requires a larger training set ($M=10,000$ versus $M = 4,000$ in \cite{GottwaldReich21a}), but involves only a single tuning parameter $\epsilon$ and repeated computation of (\ref{eq:mean}).

\begin{figure}[htbp]
\centering
\includegraphics[width = 0.48\columnwidth]{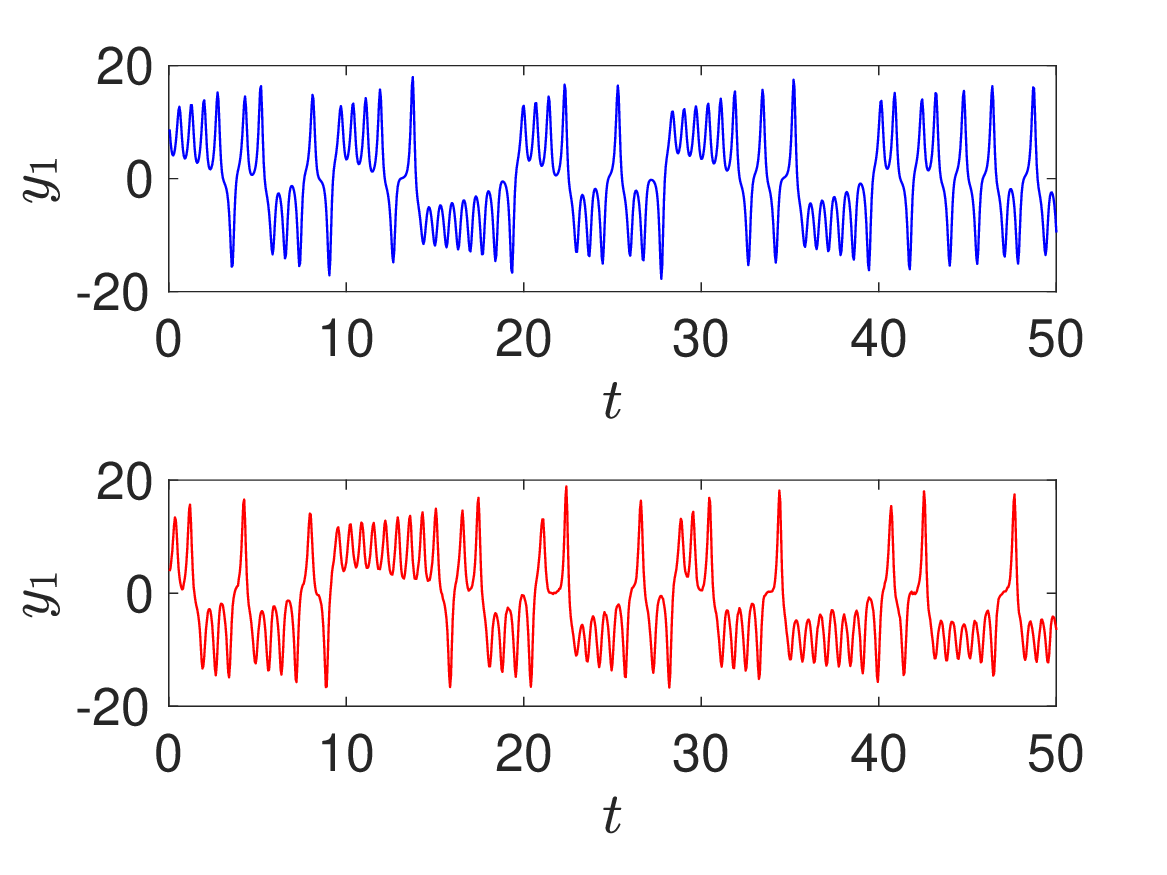}
\includegraphics[width = 0.48\columnwidth]{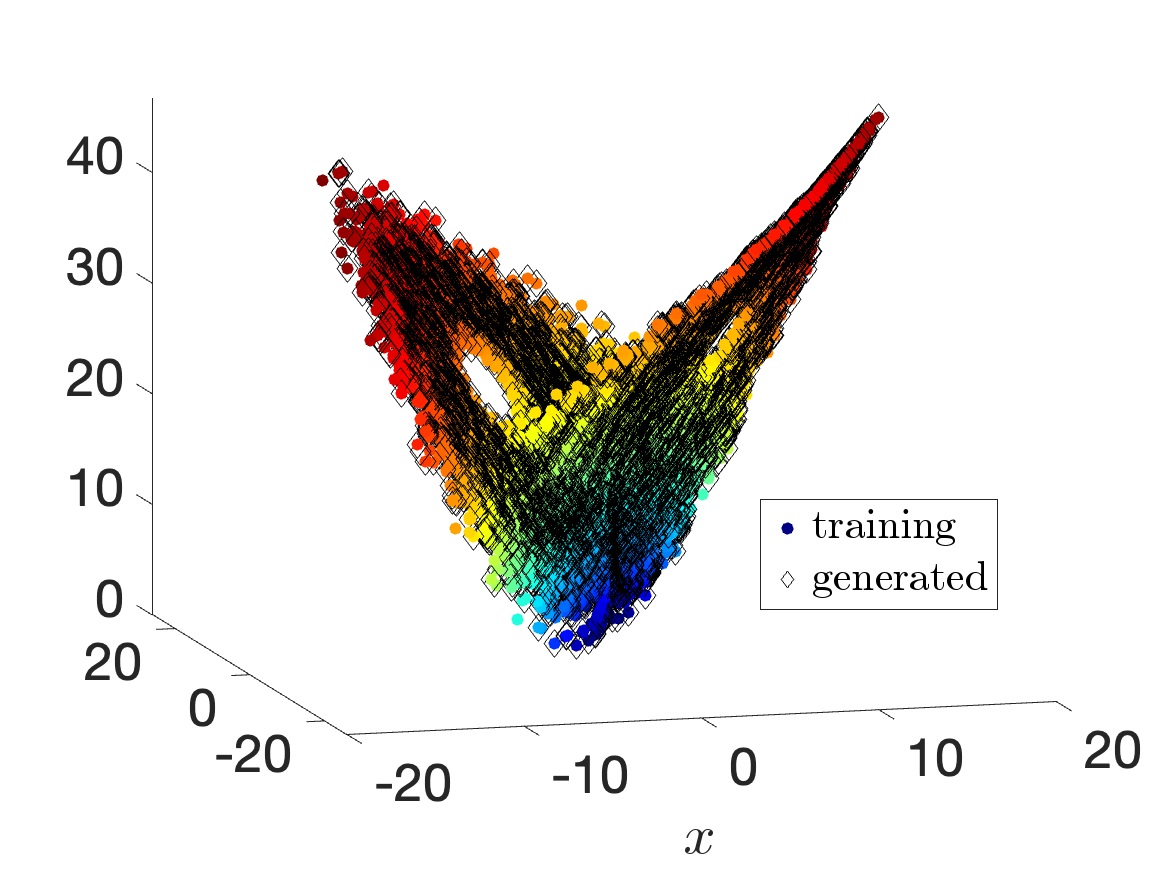}
\caption{Comparison of generated data from the Schr\"odinger bridge sampler with the original dynamics of the Lorenz-63 system \eqref{e.L63_y}. Left: Typical trajectories of $y_1$ obtained from a simulation of \eqref{e.L63_y} (top) and the generated data using conditional sampling \eqref{eq.condsamp} (bottom). Right: The corresponding attractor for the simulated and generated data.}
\label{fig.L63_gen}
\end{figure}

%%%%%%%%%%%%%%%%%%%%%%%%%%%%%%%%%%%%%
%
\section{Summary and outlook}
\label{sec:conclusion}

We have introduced a Schr\"odinger bridge based Langevin scheme for generative modeling. Our method combines sample-based Markov chain approximations with discrete-time Langevin-type sampling, yielding a non-parametric generative model that is unconditionally stable and geometrically ergodic. Even though the approach is entirely data-driven, the generative model possesses a single tuning parameter, which is the choice of the step-size $\epsilon>0$. Practical choices of $\epsilon$ depend both on the number of training samples, $M$, and the properties of the target distribution $\pi(x)$ on $\mathbb{R}^d$. We showed numerically that employing a variable bandwidth kernel, in contrast to a fixed bandwidth kernel, results in generated samples with enhanced accuracy. However, employing variable a bandthwith kernel comes with an added layer of complexity and requires delicate tuning.

In terms of practical applications, the performance of the conditional generative model was showcased through its applications to a stochastic subgrid-scale parametrization problem and the generation of trajectories of the chaotic Lorenz-63 system. In both cases, the conditional sampler relied on the availability of a single simulated trajectory, which served as the training samples. A constant bandwith kernel has been employed in all these examples.

We considered in this paper only equally weighted training samples $x^{(i)} \sim \pi(x)$. If there is a change of measure due to, for example, observed data with a given likelihood, then the measure and the couplings can be appropriately modified to take into account the non-uniform measure, i.e, weighted training samples.

Furthermore, all computational examples considered in this paper were low-dimensional, which is intrinsic to Schr\"odinger bridge approximations of the semi-group $\exp (\epsilon \mathcal{L})$ \cite{WR20}. This could be seen as a major disadvantage in comparison to the sophisticated deep neural network architectures typically used in high-dimensional SGMs \cite{diffusion1,diffusion2}. A first step to close this gap has been taken in \cite{GR24}, where a localized Schr\"odinger bridge sampler has been proposed and implemented for high-dimensional generative sampling problems. Interestingly, and worth further investigation, the localized Schr\"odinger bridge sampler can be viewed as a variant of multi-head self attention \cite{Transformer,SABP22}.

Additional future research will delve into the theoretical foundations of the proposed scheme, including its convergence rate and scalability extending results from \cite{WR20}. Finally, the proposed methodology can be extended to score generative modeling by replacing all required score functions by Schr\"odinger bridge-based approximations and to interacting particle approximations of the Fokker--Planck equation, where the score function represents now diffusion \cite{MRO20}.

%%%%%%%%%%%%%%%%%%%%%%%%%%%%%%
%
\smallskip
\smallskip
\smallskip

\paragraph{Acknowledgements.}
This work has been funded by Deutsche Forschungsgemeinschaft (DFG) - Project-ID 318763901 - SFB1294. GAG acknowledges funding from the Australian Research Council, grant DP220100931. FL and YMM acknowledge support from the US Department of Energy, Office of Advanced Scientific Computing Research, award DE-SC0023188. We thank Ricardo Baptista for pointing out the relationship to the barycentric projection.

%%%%%%%%%% Insert bibliography here %%%%%%%%%%%%%%

\bibliographystyle{abbrvnat}
\bibliography{sample.bib}

\end{document}